\documentclass[11pt]{article} 

\usepackage[utf8]{inputenc} 

\usepackage{geometry} 
\usepackage{graphicx} 

\usepackage[parfill]{parskip} 

\usepackage{booktabs} 
\usepackage{array} 
\usepackage{paralist} 
\usepackage{verbatim} 
\usepackage{subfig} 

\usepackage{amsmath}
\usepackage{amssymb}
\usepackage{amsthm}
\usepackage{bm}
\usepackage{dsfont}
\begingroup
    \makeatletter
    \@for\theoremstyle:=definition,remark,plain\do{%
        \expandafter\g@addto@macro\csname th@\theoremstyle\endcsname{%
            \addtolength\thm@preskip\parskip
            }%
        }
\endgroup

\usepackage{natbib}
\usepackage{algorithmic}
\usepackage{algorithm}

\usepackage{cite}
\usepackage{fullpage}

\newtheorem{proposition}{Proposition}
\newtheorem{assumption}{Assumption}[section]

\newtheorem{definition}{Definition}[section]
\newtheorem{theorem}{Theorem}
\theoremstyle{remark}

\newcommand{\bbR}{\mathbb{R}}

\newcommand{\bbE}{\mathbb{E}}

\newcommand{\calK}{\mathcal{K}}

\newcommand{\calO}{\mathcal{O}}
\newcommand{\calH}{\mathcal{H}}
\newcommand{\calP}{\mathcal{P}}

\newcommand{\calX}{\mathcal{X}}
\newcommand{\calZ}{\mathcal{Z}}
\newcommand{\calW}{\mathcal{W}}

\newcommand{\calR}{\mathcal{R}}
\newcommand{\one}{\mathds{1}}

\newcommand{\bW}{\mathbf{W}}

\newcommand{\bz}{\mathbf{z}}

\usepackage{fancyhdr} 
\usepackage{url}
\usepackage{todonotes}

\usepackage[nottoc,notlof,notlot]{tocbibind} 
\usepackage[titles,subfigure]{tocloft} 

\title{Joint Online Learning and Decision-making via Dual Mirror Descent}
\author{Alfonso Lobos$^1$, Paul Grigas$^1$, Zheng Wen$^2$}
\date{%
    $^1$University of California, Berkeley\\%
    $^2$Google DeepMind, Mountain view, California\\[2ex]%
    \today
}

\begin{document}
\maketitle

\begin{abstract}



We consider an online revenue maximization problem over a finite time horizon subject to lower and upper bounds on cost. At each period, an agent receives a context vector sampled i.i.d. from an unknown distribution and needs to make a decision adaptively. The revenue and cost functions depend on the context vector as well as some fixed but possibly unknown parameter vector to be learned. We propose a novel offline benchmark and a new algorithm that mixes an online dual mirror descent scheme with a generic parameter learning process.  When the parameter vector is known, we demonstrate an $O(\sqrt{T})$ regret result as well an $O(\sqrt{T})$ bound on the possible constraint violations. When the parameter is not known and must be learned, we demonstrate that the regret and constraint violations are the sums of the previous $O(\sqrt{T})$ terms plus terms that directly depend on the convergence of the learning process.
\end{abstract}

\section{Introduction \label{Sec:Intro}}

We consider an online revenue maximization problem over a finite time horizon, subject to multiple lower and upper bound cost constraints. At each time period, an agent receives a context vector and needs to make a real-time decision. After making a decision, the agent earns some revenue and also incurs multiple costs, which may alternatively be interpreted as the consumption of multiple resources. Unlike the typical situation in online optimization and learning (see, e.g., \citep{hazan2019introduction}), the agent has estimates of the revenue and cost functions available before making a decision. These estimates are updated sequentially via an exogenous learning process. Thus, there are three major challenges in this online learning and decision-making environment:  \textit{(i)} balancing the trade-off between revenue earned today and ensuring that we do not incur too many costs too early, \textit{(ii)} ensuring that enough costs are incurred to meet the lower bound constraints over the full time horizon, and \textit{(iii)} understanding the effects of the parameter learning process.

Examples of this online learning and decision-making setup occur in revenue management, online advertising, and online recommendation.
In revenue management, pricing and allocation decisions for goods and services with a limited supply need to be made in real-time as customer arrivals occur \citep{bertsimas2003revenue,boyd2003revenue}. This setup is also prevalent in online advertising, for example, in the case of a budget-constrained advertiser who bids in real-time auctions in order to acquire valuable impressions.  Importantly, each arrival typically has associated a feature vector to it, for example, the cookie history of a user to which an ad can be shown. How that feature may relate to useful quantities, e.g., the probability of a user clicking an ad, may need to be learned. Finally, our setting considers lower bounds on cost since in many industries minimum production or marketing goals are desired.

\subsection{Contributions} 
Our contributions may be summarized as follows:
\begin{enumerate}
    \item We propose a novel family of algorithms to tackle a joint online learning and decision making problem. Our setting considers both lower and upper bound constraints on cost functions and does not require strong assumptions over the revenue and cost functions used, such as convexity. Our work can be understood as an extension of an online optimization problem in which we may also need to learn a generic parameter. Furthermore, our work can be considered as in a 1-lookup ahead setting as the agent can observe the current context vector before taking a decision.
    \item We propose a novel benchmark to compare the regret of our algorithm. Our benchmark is considerably stricter in comparison to the expected best optimal  solution in hindsight. Our benchmark is specially well suited to handle settings with ``infeasible sequence of context vector arrivals" for which it is impossible to satisfy the cost constraints. We construct a dual problem which upper bounds the benchmark and we demonstrate how to efficiently obtain stochastic subgradients for it. 
    \item In the case when no ``generic parameter learning'' is needed, we prove that the regret of our algorithm is upper bounded by $\mathcal{O}(\sqrt{T})$ under a Slater condition. Given the generic setup of our problem, this is a contribution on the field of online optimization. In the general case, our regret decomposes between terms upper bounded by $\mathcal{O}(\sqrt{T})$ and terms coming from the convergence of the generic parameter learning.
    \item We prove that the solution given by our algorithm may violate any given lower bound constraint by at most $O(\sqrt{T})$ in the online optimization case, while upper bounds are always satisfied by construction. Therefore, our methodology is asymptotically feasible in the online optimization case \citep{liakopoulos2019cautious}.
    \item We demonstrate that our algorithm is effective and robust as compared to a heuristic approach in a bidding and allocation problem with no generic parameter learning in online advertising. Additionally, we study the effects of different generic parameter learning strategies in a linear contextual bandits problem with bounds on the number of actions taken.
\end{enumerate}

\subsection{Related Work}

The problem of online revenue maximization under feasibility constraints has been mostly studied under the lens of online convex optimization \citep{hazan2019introduction}. While first studied on resource allocation problems under linear constraints \citep{mehta2007adwords,devanur2011near}, arbitrary convex revenue and cost functions are used today. Of major importance is the nature of the data arrivals. Typically, data has been assumed to be received in an adversarial \citep{devanur2011near,chen2017online} or an i.i.d. manner \citep{wei2020online,balseiro2020dual}, with the data being sampled from an unknown distribution in the latter case. Subgradient methods based on primal-dual schemes have gained attraction \citep{devanur2011near, jenatton2016adaptive, chen2017online, yuan2018online} as they avoid taking expensive projection iterations by penalizing the constraints through duality (either Lagrangian or Fenchel). Consequently, it is important to study both regret and the worst possible constraint violation level.

In the adversarial setting, regret is typically measured against the best-static decision in hindsight and algorithms achieving $O(\sqrt{T})$ regret, which is optimal in the adversarial setting, and different level of constraint violations levels have been achieved \citep{mahdavi2012trading, jenatton2016adaptive,chen2017online,yuan2018online}. On the i.i.d. setting and under linear constraints, \citet{balseiro2020dual} obtains an $O(\sqrt{T})$ regret bound and no constraint violation by algorithm construction (since they consider linear constraints with no lower bounds). Since they consider a 1-lookup ahead setting with i.i.d.\ arrivals, \citet{balseiro2020dual} use the best dynamic solution in hindsight as a benchmark, which is a considerably stricter benchmark than the commonly used best static solution. Our joint online learning and optimization model and algorithmic strategy builds upon the online optimization model and dual Mirror Descent approach for resource allocation presented by \citet{balseiro2020dual}. Note that our first contribution, the incorporation of arbitrary revenue and cost functions, was simultaneously obtained by the same set of authors on \citet{balseiro2020best}.

A stream of literature studying a similar problem to ours is ``Bandits with Knapsacks'' (BwK) and extensions of it. In BwK, an agent operates over $T$ periods of time. At each period, the agent chooses an action, also known as an arm, from a \textit{finite} set of possible action and observes a reward and a cost vector. As us, the agent would like to satisfy global cost constraints. BwK is studied both in an adversarial and i.i.d. settings, but here we only emphasize on the latter (see \citet{immorlica2019adversarial} for the adversarial case). Assuming concave reward functions, \citet{agrawal2014bandits} proposes an Upper-Confidence Bound type of algorithms which achieves sublinear rates of regret and constraint violations. \citet{badanidiyuru2018bandits} proposes a primal-dual algorithm to solve BwK with has a sublinear regret, and by algorithm construction, their cost constraints are always satisfied. Our job extends on this literature stream in the following ways. 1. We allow an \textit{arbitrary} action space and reward and cost functions. 2. Our proposed benchmark is stricter than the best expected dynamic policy. 3. The novel joint learning and decision-making setting proposed here.



\subsection{Notation}

We use $\bbR_+^N := \{ x \ge 0: x \in \bbR^N \}$, $\bbR_-^N := \{ x \le 0: x \in \bbR^N \}$, and $[N] := \{ 1,\dots, N \}$ with $N$ being any integer. For any $x \in \bbR^N$ and $y \in \bbR^N$, $x \odot y := (x_1 y_1, \dots, x_N y_N)$ and $x^T y := \sum_{i=1}^n x_i y_i$ representing the element-wise and dot products between vectors of same dimension. We use $x \in A$ to represent that $x$ belongs to set $A$, and $(x^1, \dots, x^N) \in A^1 \times \dots \times A^N$ represents $x^i \in A^i$ for all $i \in [n]$. We reserve capital calligraphic letters to denote sets.  For any $x \in \bbR^N$, $[x]_+ := (\max \{x_1, 0 \}, \dots, \max \{x_N, 0 \})$ and $\mathds{1}(x \in A) := 1$ if $x \in A$ and $0$ otherwise. We use $\lVert \cdot \rVert$ to represent a norm operator, and in particular, for any $x \in \bbR^N$ we use $\lVert x \rVert_1 := \sum_{i=1}^N \lvert x_i \rvert$, $\lVert x \rVert_2 := \sqrt{\sum_{i=1}^N x_i^2}$, and $\lVert x \rVert_{\infty} = \max_{i \in [N]}$ $\lvert x_i \rvert$. For any real-valued convex function $f: \calX \rightarrow \bbR$, we say that $g$ is a subgradient of $f(\cdot)$ at $x \in \calX$ if $f(y) \ge f(x) + g^T (y-x)$ holds for all $y \in \calX$, and use $\partial f(x)$ to denote the set of subgradients of $f(\cdot)$ at $x$.

\section{Preliminaries and Algorithm}
We are interested in a real-time decision-making problem over a time horizon of length $T$ involving three objects: {\em (i)} $z^t \in \calZ \subseteq \bbR^d$, the decision to be made at time $t$, {\em (ii)} $\theta^* \in \Theta \subseteq \bbR^p$, a possibly unknown parameter vector describing the revenue and cost functions that may need to be learned, and {\em (iii)} $w^t \in \calW \subseteq \bbR^m$, a context vector received at prior to making a decision at time $t$. These three objects describe the revenue and cost functions that are central to the online decision-making problem. In particular, let $f(\cdot;\cdot,\cdot): \calZ \times \Theta \times \calW \rightarrow \bbR$ denote the revenue function and let $c(\cdot;\cdot,\cdot): \calZ \times \Theta \times \calW \rightarrow \bbR^K$ denote the collection of $K$ different cost functions. We assume that these functions are bounded, namely for the true revenue function it holds that $\sup_{z \in \calZ, w \in \calW} f(z;\theta^*, w) \le \bar{f}$ with $\bar{f}>0$ and for the cost functions it holds that $\sup_{z \in \calZ, \theta \in \Theta, w \in \calW} \lVert c(z;\theta, w) \rVert_\infty  \le \bar{C}$ with $\bar{C}>0$.

At each time period $t$, first $w^t$ is revealed to the decision maker and is assumed to be drawn \textrm{i.i.d} from an unknown distribution $\calP$ over $\calW$.  For example, if $\calW$ is a finite set, then $w^t$ could represent the scenario being revealed at time $t$. We assume that once the decision maker observes a context vector $w^t \in \calW$, then it also observes or otherwise have knowledge of the parametric forms of revenue and cost functions $f(\cdot;\cdot, w^t): \calZ \times \Theta \rightarrow \bbR $ and $c(\cdot;\cdot, w^t): \calZ \times \Theta \rightarrow \bbR^K$.
Although the true parameter $\theta^\ast$ may be unknown to the decision maker at time $t$, whenever a decision $z^t \in \calZ$ is made the revenue earned is equal to $f(z^t, \theta^\ast, w^t)$ and the vector of cost values incurred is equal to $c(z^t, \theta^\ast, w^t)$. 

In an ideal but unrealistic situation, the decision planner would be able to observe the sequence $(w^1,\ldots,w^T)$ of future context vector arrivals and would set the decision sequence $(z^1, \ldots, z^T)$ by solving the full observability (or hindsight) problem:
\begin{align}
    (O): \ \ &\underset{(z^1, \ldots, z^T) \in \calZ^T}{\max} \ \ \sum_{t = 1}^T f(z^t;\theta^*,w^t)   \nonumber \\
    &\mathrm{ s.t.} \ \ T \alpha \odot b \le \sum_{t = 1}^T c(z^t;\theta^*,w^t) \le Tb   \label{Eq:OrigProblem}
\end{align}
\noindent where $b \in \bbR_{++}^K$, and $\alpha \in [-1,1)^K \cup \{ - \infty \}$ with $\alpha_k = -\infty$ meaning that no lower bounds are present for coordinate $k$. Define $\underline{b} := \min_{k \in [K]}b_k$ and $\bar{b} := \max_{k \in [K]}b_k$, and we assume that $\underline{b} > 0$.
The vector $b$ can be thought as a resource or budget vector proportional to each period. Then, \eqref{Eq:OrigProblem} is a revenue maximization problem over the time horizon $T$ with lower and upper cost constraints. Setting $-1$ as the lower bound for $\alpha_k$ for all $k \in [K]$ is an arbitrary choice only affecting some of the constants in the regret bounds we prove. 

Before providing more details on the dynamics of the problem and our proposed algorithm, we introduce a novel benchmark to evaluate the performance/regret of our algorithm. The primary need for a new benchmark in our context is that the generality of our problem leads to feasibility issues.
Indeed, for some combinations of context vector arrivals, problem \eqref{Eq:OrigProblem} may be infeasible due the presence of both lower and upper bound constraints as well as the fact that the costs functions are generic. 
We now define an offline benchmark as follows. A natural benchmark to consider is the \emph{expected} optimal value of \eqref{Eq:OrigProblem}. 
However, as long as there is any positive probability of \eqref{Eq:OrigProblem} being infeasible, then this benchmark will be $-\infty$, which will lead to trivial regret bounds. Thus, to avoid such trivialities, we consider a benchmark that interpolates between the expected optimal value of \eqref{Eq:OrigProblem} and a deterministic problem that replaces the random revenue and cost functions with their expected values. In particular, let $\gamma \in [0,1]$ denote this interpolation parameter. For any $z \in \calZ$, $\theta \in \Theta$, $w^\prime \in \calW$, $w \sim \calP$, and $\gamma \in [0,1]$ we define:
\begin{align*}
\mathrm{rev}(z;\theta,w',\gamma) &:= (1-\gamma)f(z;\theta,w')  + \gamma\mathbb{E}_{\calP}[f(z;\theta,w)] \\
\mathrm{cost}(z;\theta,w',\gamma) &:= (1-\gamma) c(z;\theta,w') + \gamma \mathbb{E}_{\calP}[c(z;\theta,w)].
\end{align*}
Let $\calP^T := \calP \times \dots \times \calP$ denote a product distribution of length $T$, i.e., the distribution of $(w^1, \ldots, w^T)$. 
Now, for any $\gamma \in [0,1]$, let us define
\begin{align}
     \mathrm{OPT}(\calP,\gamma) : = \nonumber 
     \mathbb{E}_{\calP^T}
\left\lbrack
\begin{array}{c l}	
     \underset{z^t \in \calZ: t \in [T]}{\max} \ \ \sum_{t=1}^T \mathrm{rev}(z^t;\theta^*,w^t,\gamma)  &  \\
     \textrm{s.t.} \ \   T \alpha \odot b \le \sum_{t=1}^T \mathrm{cost}(z^t;\theta^*,w^t,\gamma) \le Tb 
\end{array}
\right\rbrack \nonumber 
\end{align}
and let us further define
\begin{equation}
    \mathrm{OPT}(\calP) := \max_{\gamma \in [0,1]}  \mathrm{OPT}(\calP,\gamma). \label{Eq:OptOffline}
\end{equation}
Note that $\mathrm{OPT}(\calP,0)$ is exactly the expected optimal value of the hindsight problem \eqref{Eq:OrigProblem}. 
On the other hand, $\mathrm{OPT}(\calP,1)$ corresponds to a deterministic approximation of \eqref{Eq:OrigProblem} that replaces all random quantities with their expectations and is typically a feasible problem.
Then, we can understand $\gamma \in [0,1]$ as an interpolation parameter between the more difficult hindsight problem $\mathrm{OPT}(\calP,0)$ and the expectation problem $\mathrm{OPT}(\calP,1)$. Importantly, the benchmark we consider is $\mathrm{OPT}(\calP)$, which considers the \emph{worst case} between these two extremes. It is possible to have $\mathrm{OPT}(\calP) = \mathrm{OPT}(\calP,0)$, $\mathrm{OPT}(\calP) = \mathrm{OPT}(\calP,1)$, $\mathrm{OPT}(\calP) = \mathrm{OPT}(\calP,\gamma)$ for some $\gamma \in (0,1)$, and $\mathrm{OPT}(\calP) = -\infty$. It is also possible to have a unique $\gamma$ that maximizes $\mathrm{OPT}(\calP,\gamma)$ as well as infinitely many such maximizers. Examples of all of these possibilities are included in the supplementary materials. 

\subsection{Joint Learning and Decision-making Dynamics and Regret Definition}
Now we describe the dynamics of our joint online learning and decision-making problem as well as a generic ``algorithmic scheme.'' In Section \ref{sec:dual_md}, we give a complete algorithm after building up the machinery of dual mirror descent.
Let $\mathcal{I}^t := (z^t, \theta^t, w^t, f^t(z^t;\theta^\ast,w^t), c(z^t;\theta^*,w^t))$ denote the information obtained during period $t$, and let $\mathcal{H}^t := (\mathcal{I}^1, \ldots, \mathcal{I}^T)$ denote the complete history up until the end of period $t$. Note that it is assumed that the decision planner observes the exact incurred cost value vector $c(z^t;\theta^*,w^t)$, but there is a possibility of including additional randomness in the observed revenue. In particular, the observed revenue $f^t(z^t;\theta^*,w^t)$ satisfies $f^t(z^t;\theta^*,w^t) = f(z^t;\theta^*,w^t) + \xi_t$ where $\xi_t$ is a mean zero random variable that is allowed to depend on $w^t$ but is independent of everything else.  

Let $A_{\theta}$ refer to a generic learning algorithm and let $A_{z}$ refer to a generic decision-making algorithm. Then, at any time period $t$, the decision planner sets
\begin{align}
    & \theta^t = A_{\theta}\left(\mathcal{H}^{t-1} \right) \nonumber, \\
    & z^t = A_{z}\left(f(\cdot;\theta^t,w^t), c(\cdot;\theta^t,w^t), \mathcal{H}^{t-1} \right) \label{Eq:Calc_zt_thetat}
\end{align}
We refer to $(A_{z}, A_{\theta})$ as $A$ when no confusion is possible. Note that an important special case is when $A_{\theta}$ outputs $\theta^\ast$ for all inputs, which is the case where $\theta^\ast$ is known.
Algorithm \ref{Alg:AbsAlgorithm}, which alternates between an online learning step using $A_\theta$ and an online decision-making step using $A_z$, specifies the precise sequence of events when using the generic algorithm $A$.
Recall that $\bar{C} := \sup_{(z,\theta,w) \in \calZ \times \Theta \times \calW}$ $ \lVert c(z;\theta,w) \rVert_{\infty}$, which is a constant that we will use as the minimum allowable remaining cost budget. For simplicity we assume that the constant $\bar{C}$ is available although we can easily replace it with an available upper bound.

\floatname{algorithm}{Algorithm} 
\begin{algorithm} 
\caption{Generic Online Learning and Decision-making Algorithmic Scheme} \label{Alg:AbsAlgorithm}
\begin{algorithmic}
\STATE {\textbf{Input:}} Initial estimate $\theta^1 \in \Theta$, and remaining cost budget vector $b^1 \gets T b$.
\FOR{$t = 1,\dots,T$}
\STATE 1. Update $\theta^t \gets A_{\theta}\left(\mathcal{H}^{t-1} \right)$. 
\STATE 2. Receive $w^t \in \mathcal{W}$, which is assumed to be drawn from an unknown distribution $\calP$ and is independent of $\calH^{t-1}$.
\STATE 3. Set $z^t \gets A_{z}\left(f(\cdot;\theta^t,w^t), c(\cdot;\theta^t,w^t), \mathcal{H}^{t-1} \right)$.
\STATE 4. Update remaining cost budget $b^{t+1} \gets b^t - c(z^t;\theta^*,w^t)$, and earn revenue $f^t(z^t;\theta^*,w^t)$. 
\STATE 5. If $b_k^{t+1}<\bar{C}$ for any $k \in [K]$, \textbf{break}.
\ENDFOR
\end{algorithmic}
\end{algorithm}
Note that Steps 4.\ and 5.\ of Algorithm \ref{Alg:AbsAlgorithm} ensure that the total cost incurred is always less than or equal to $b T$, which ensures that the upper bound constraints in \eqref{Eq:OrigProblem} are always satisfied, while there is a chance that some lower bound constraints may not be satisfied.
These steps make our later theoretical analysis simpler, but less conservative approaches can be used, for example allowing the algorithm to exceed $bT$ once.

Define $R(A|\calP) = \mathbb{E}_{\calP^T} \left\lbrack \sum_{t=1}^T f(z^t;\theta^*,w^t) \right\rbrack$ as the expected revenue of algorithm $A$ over distribution $\calP^T$, where $z^t$ is computed as in \eqref{Eq:Calc_zt_thetat}.  We define the regret of algorithm $A$ as $\mathrm{Regret}(A|\calP) :=$ $\mathrm{OPT}(\calP) - R(A|\calP)$.
Since the probability distribution $\calP$ is unknown to the decision maker, our goal is to design an algorithm $A$ that works well for any distribution $\calP$. That is, we would like to obtain a good distribution free regret bound.

\subsection{Dual Problem and Dual Mirror Descent Algorithm}\label{sec:dual_md}
We now consider a Lagrangian dual approach that will naturally lead to a dual mirror descent algorithm. Let $\lambda \in \bbR^K$ denote a vector of dual variables, and we define the set of feasible dual variables as $\Lambda := \{ \lambda \in \bbR^K:$ $ \lambda_k \ge 0 $ $\textrm{for all}$ $k$ $\textrm{with}$ $\alpha_k = -\infty\}$. For any triplet $(\lambda,\theta, w) \in \Lambda \times \Theta \times \calW$ define
 \begin{align*}
    \varphi(\lambda;\theta,w) & := \underset{z \in \calZ}{\max} \,   f(z;\theta,w) - \lambda^T c(z;\theta,w) \\
    z(\lambda;\theta,w) & :\in \arg\max_{z \in \calZ} f(z;\theta,w) - \lambda^T c(z;\theta,w),
\end{align*} 
and for any $(\lambda, \theta) \in \Lambda \times \Theta$ define
\begin{align*}
    p(\lambda) & : =  \sum_{k \in {[K]}} b_k ([\lambda_k]_+ - \alpha_k [-\lambda_k]_+) \\
    D(\lambda;\theta) & := \mathbb{E}_\calP[\varphi(\lambda;\theta,w)] + p(\lambda).
\end{align*} 
This works assumes that $z(\lambda;\theta,w)$ exists and can be efficiently computed for any $(\lambda, \theta, w) \in (\Lambda, \Theta, \calW)$. Furthermore, in case there are multiple optimal solutions corresponding to $\varphi(\lambda;\theta,w)$ we assume that the subroutine for computing $z(\lambda;\theta,w)$ breaks ties in a deterministic manner.
We call $D(\cdot;\theta)$ the dual function given parameters $\theta$, which is a key component of the analysis and algorithms proposed in this work. In particular, we first demonstrate in Proposition \ref{Pro:WeakDuality} that $D(\cdot;\theta^*)$ can be used to obtain an upper bound on $\mathrm{OPT}(\calP)$.

\begin{proposition}\label{Pro:WeakDuality}
For any $\lambda \in \Lambda$, it holds that $\mathrm{OPT}(\calP)$ $\le$ $T D(\lambda;\theta^*)$.
\end{proposition} 

Next, Proposition \ref{Pro:StochSubg} demonstrates that a stochastic estimate of a subgradient of $D(\cdot;\theta)$ can be easily obtained during the sequence of events described in Algorithm \ref{Alg:AbsAlgorithm}.

\begin{proposition}\label{Pro:StochSubg}
Let $\lambda \in \Lambda$, $\theta \in \Theta$, and $w \in \calW$ be given. Define $\tilde{g}(\lambda;\theta,w) \in \bbR^K$ by $\tilde{g}_k(\lambda;\theta,w) := -c_k(z(\lambda;\theta,w);\theta,w) + b_k \left(\one(\lambda_k \ge 0) + \alpha_k\one(\lambda_k < 0) \right)$ for all $k \in [K]$. Then, if $w \sim \calP$, it holds that $\tilde{g}(\lambda;\theta,w)$ is a stochastic subgradient estimate of $D(\cdot;\theta)$ at $\lambda$, i.e., $ \mathbb{E}_\calP[\tilde{g}(\lambda;\theta,w)] \in \partial_{\lambda} D(\lambda;\theta)$.
\end{proposition}

We are now ready to describe our dual mirror descent algorithm.
Let $h(\cdot): \Lambda \rightarrow \bbR$ be the reference function for mirror descent, which we assume is $\sigma_1$-strongly convex in the $\ell_1$-norm, i.e., for some $\sigma_1 > 0$ it holds that $h(\lambda) \ge h(\lambda')$ $+$ $\langle \nabla h(\lambda'), \lambda - \lambda' \rangle $ $+$ $\tfrac{\sigma_1}{2} \lVert \lambda - \lambda' \rVert_1^2$  for any $\lambda, \lambda'$ in $\Lambda$. Also, we assume that $h(\cdot)$ is a separable function across components, i.e., it satisfies 
$h(\lambda) = \sum_{k = 1}^K$ $h_k(\lambda_k)$ where $h_k(\cdot) : \bbR \to \bbR$ is a convex univariate function for all $k \in [K]$. Define $V_h(\lambda,\lambda^\prime) : = h(\lambda)-h(\lambda^\prime)-\nabla h(\lambda^\prime)^T(\lambda-\lambda^\prime)$, the Bregman divergence using $h(\cdot)$ as the reference function.

Algorithm \ref{Alg:DualMirrorLearning} presents the main algorithm of this work. Algorithm \ref{Alg:DualMirrorLearning} is a specific instance of the more general algorithmic scheme, presented in Algorithm \ref{Alg:AbsAlgorithm}, where we fill in the generic decision making subroutine $A_z$ with a dual stochastic mirror descent \citet{hazan2019introduction,beck2003mirror} step with respect to the current estimate of the dual problem $ \min_{\lambda \in \Lambda}$ $D(\lambda;\theta^t)$. Note that the learning subroutine $A_\theta$ is left as a generic subroutine; the regret bounds that we prove in Section \ref{sect:theory} hold for any learning algorithm $A_\theta$ and naturally get better when $A_\theta$ has better convergence properties.

\begin{algorithm} 
\caption{Online Learning and Decision-making via Dual Mirror Descent} \label{Alg:DualMirrorLearning}
\begin{algorithmic}
\STATE {\textbf{Input:}} Initial estimate $\theta^1 \in \Theta$, remaining cost budget vector $b^1 = T b$, and initial dual solution $\lambda^1$.
\FOR{$t = 1,\dots,T$}
\STATE 1. Update $\theta^t \gets A_{\theta}\left(\mathcal{H}^{t-1} \right)$. 
\STATE 2. Receive $w^t \in \mathcal{W}$, which is assumed to be drawn from an unknown distribution $\calP$ and is independent of $\calH^{t-1}$.
\STATE 3. Make primal decision $z^t \gets z(\lambda^t;\theta^t,w^t)$, i.e.,
\begin{equation*}
z^t \in \arg\max_{z \in \calZ} f(z;\theta^t,w^t) - (\lambda^t)^T c(z;\theta^t,w^t).
\end{equation*}
\STATE 4. Update remaining cost budget $b^{t+1} \gets b^t - c(z^t;\theta^*,w^t)$, and earn revenue $f^t(z^t;\theta^*,w^t)$. 
\STATE 5. If $b_k^{t+1}<\bar{C}$ for any $k \in [K]$, \textbf{break}.
\STATE 6. Obtain dual stochastic subgradient $\tilde{g}^t$ where $\tilde{g}_k^t  $ $\gets$ $- c_k(z^t;\theta^t,w^t) + b_k \left(\one(\lambda_k \ge 0) + \alpha_k\one(\lambda_k < 0) \right) $ for all $k \in [K]$.
\STATE 7. Choose ``step-size" $\eta_t$ and take dual mirror descent step 
\begin{equation*}
\lambda^{t+1} \gets \arg \min_{\lambda \in \Lambda} \, \lambda^T \tilde{g}^t + \tfrac{1}{\eta_t} V_h(\lambda, \lambda^t).
\end{equation*}
\ENDFOR
\end{algorithmic}
\end{algorithm}

Note that Proposition \ref{Pro:StochSubg} ensures that $\tilde{g}^t$ from Step 6.\ of Algorithm \ref{Alg:DualMirrorLearning} is a stochastic subgradient of $D(\cdot;\theta^t)$ at $\lambda^t$. The specific form of the mirror descent step in Step 7.\ depends on the reference function $h(\cdot)$ that is used. A standard example is the Euclidean reference function, i.e., $h(\cdot) := \frac{1}{2} \lVert \cdot \rVert_2^2$, in which case Step 7.\ is a projected stochastic subgradient descent step. Namely, $\lambda_k^{t+1} $ $\gets$ 
$[\lambda_k^t -\eta \tilde{g}_k^t]_+$ for all $k \in [K]$ with $\alpha_k = -\infty$ and $\lambda_k^{t+1} $ $\gets$ $\lambda_k^t - \eta \tilde{g}_k^t$ otherwise. A simple extension of this example is $h(\lambda) := \lambda^TQ\lambda$ for some positive definite matrix $Q$. When no lower bounds are present, i.e., $\alpha_k = -\infty$ for all $k \in [K]$,  we can use an entropy-like reference function $h(\lambda) := \sum_{k \in [K]}$ $\lambda_k\log(\lambda_k)$ wherein Step 7.\ becomes a multiplicative weight update $\lambda_k^t \gets \lambda^t \exp(-\eta_t \tilde{g}_k^t)$  \citet{arora2012multiplicative}. Finally, note that since the reference function is component wise separable, one may use a different type of univariate reference function for different components. 


While Algorithm \ref{Alg:DualMirrorLearning} fills in the gap for $A_z$ using mirror descent, the learning algorithm $A_\theta$ in Step 1.\ is still left as generic and there are a range of possibilities that one might consider depending on the specific problem being addressed.
For example, considering only the revenue function for simplicity, suppose that there is a feature map $f': \calZ \times \calW \rightarrow \bbR^p$ such that $f(z;\theta,w) = f'(z;w)^T \theta$ for $(z,\theta,w) \in \calZ \times \Theta \times \calW$ and we observe both $f(z^t;\theta^\ast,w^t)$ and $f'(z^t;w^t)$ at time $t$. Then, one could use $(f^s(z^s;\theta^*,w^s), f'(z^s;w^s))_{s=1}^{t-1}$ to fit a linear model (possibly with regularization) for implementing $A_\theta$ at time $t$.
Depending on the underlying structure of the problem and randomness of the data arrivals, the previous methods may not converge to $\theta^*$. Different ways of applying Step 1. are shown for a linear contextual bandits problem in Section \ref{Sec:ExsAndExps}.
The performance of the different implementations vary drastically depending on the underlying randomness of the data arrivals.

\section{Regret Bound and Related Results}\label{sect:theory}
In this section, we present our main theoretical result, Theorem \ref{Thm:Master}, which shows regret bounds for Algorithm \ref{Alg:DualMirrorLearning}. In particular, the regret of Algorithm \ref{Alg:DualMirrorLearning} can be decomposed as the summation of two parts:  {\em (i)} the terms that appear when $\theta^*$ is known, which emerge from the properties of the Mirror Descent algorithm and can be bounded sublinearly as $\calO(\sqrt{T})$, and {\em (ii)} terms that naturally depend on the convergence of the learning process towards $\theta^\ast$. We also discuss the proof strategy for Theorem \ref{Thm:Master}. Finally, for each lower bound constraint in \eqref{Eq:OrigProblem}, we prove that our algorithm may violate this lower bound by at most $\calO(\sqrt{T})$ plus terms that depend on how $\theta^t$ converges to $\theta^\ast$.


\subsection{Regret Bound}

Before presenting our main theorem, we need to establish a few more ingredients of the regret bound. First, we present Assumption \ref{Ass:DualVarBounded}, which can be thought of as a boundedness assumption on the dual iterates.
\begin{assumption}[Bounded Dual Iterates]\label{Ass:DualVarBounded}
There is an absolute constant $C_h > 0$ such that the dual iterates $\{\lambda^t\}$ of Algorithm \ref{Alg:DualMirrorLearning} satisfy $\mathbb{E} \left[ \lVert \nabla h(\lambda^t) \rVert_\infty \right] \le C_h$ for all $t \in [T]$.
\end{assumption}
Note that, in the Euclidean case where $h(\lambda) = \frac{1}{2} \lVert \lambda \rVert_2^2$, we have $\nabla h(\lambda) = \lambda$ and therefore Assumption \ref{Ass:DualVarBounded} may be thought of as a type of boundedness condition. After stating our regret bound, we present a sufficient condition for Assumption \ref{Ass:DualVarBounded}, which involves only the properties of the problem and not the iterate sequence of the algorithm.


Now, recall that $\mathcal{H}^t$  can be understood as all the information obtained by Algorithm \ref{Alg:DualMirrorLearning} up to period $t$.  Then, Step 4. of Algorithm \ref{Alg:DualMirrorLearning} is intrinsically related to the following stopping time with respect to $\calH^{t-1}$.
\begin{definition}[Stopping time]\label{Def:StopTimeTauA} Define $\tau_{A}$ as the minimum between $T$ and the smallest time $t$ such that there exists $k \in [K]$ with
$\sum_{t=1}^{\tau_{A}} c_k(z^t;\theta^*,w^t) + \bar{C} > b_k T$.
\end{definition}

Finally, recall that we defined constants $\bar{f} > 0$, $\bar{C} > 0$, $\underline{b} > 0$ and $\bar{b} > 0$ such that $\sup_{z \in \calZ, w \in \calW} f(z;\theta^*, w) \le \bar{f}$, $\sup_{z \in \calZ, \theta \in \Theta, w \in \calW} \lVert c(z;\theta, w) \rVert_\infty  \le \bar{C}$, $\underline{b} := \min_{k \in [K]}b_k$ and $\bar{b} := \max_{k \in [K]}b_k$. Also, $\sigma_1$ refers to the strong convexity constant of $h(\cdot)$. We are now ready to state Theorem \ref{Thm:Master}, which presents our main regret bound.

\begin{theorem}\label{Thm:Master}
Let $A$ denote Algorithm \ref{Alg:DualMirrorLearning} with a constant ``step-size'' rule $\eta_t \gets \eta$ for all $t \geq 1$ where $\eta > 0$. Suppose that Assumption \ref{Ass:DualVarBounded} holds.
Then, for any distribution $\calP$ over $w \in \calW$, it holds that $\mathrm{Regret}(A|\calP) ~\leq~ \Delta_{\mathrm{DM}} + \Delta_{\mathrm{Learn}}$ where 
\begin{align*}
     \Delta_{\mathrm{DM}} :=~  &\frac{ 2 (\bar{C}^2 +\bar{b}^2)}{\sigma_1} \eta \mathbb{E}[\tau_A] + \frac{1}{\eta} V_h(0,\lambda^1)  + \frac{\bar{f}}{\underline{b}} \left( \bar{C} + \frac{C_h + \lVert \nabla h(\lambda^1) \rVert_\infty}{\eta }  \right) \\
    \Delta_{\mathrm{Learn}} :=~ &\mathbb{E} \left\lbrack \sum_{t=1}^{\tau_A} (c(z^t;\theta^*,w^t) - c(z^t;\theta^t,w^t))^T\lambda^t  \right\rbrack  + \frac{\bar{f}}{\underline{b}} \left\lVert \mathbb{E} \left[  \sum_{t=1}^{\tau_A} c(z^t;\theta^*,w^t) - c(z^t;\theta^t,w^t) \right] \right\rVert_\infty.
\end{align*}
\end{theorem}




Theorem \ref{Thm:Master} states that the regret of Algorithm \ref{Alg:DualMirrorLearning} can be upper bounded by the sum of two terms:  {\em (i)} a quantity $\Delta_{\mathrm{DM}}$ that relates to the properties of the decision-making algorithm, dual mirror descent, and {\em (ii)} a quantity $\Delta_{\mathrm{Learn}}$ that relates to the convergence of the learning algorithm $A_\theta$. It is straightforward to see that setting $\eta \gets \gamma/\sqrt{T}$ for some constant parameter $\gamma>0$ implies that $\Delta_{\mathrm{DM}}$ is $O(\sqrt{T})$. In the pure online optimization case, $\theta^\ast$ is known and hence $\theta^t = \theta^\ast$ for all $t \in [T]$ yielding $\Delta_{\mathrm{Learn}} = 0$. Thus, using $\eta \gets \gamma/\sqrt{T}$ in the pure online optimization case yields $\mathrm{Regret}(A|\calP) \leq O(\sqrt{T})$ and extends results presented by \citet{balseiro2020dual}.
More generally, $\Delta_{Learn}$ depends on the convergence of $\theta^t$ to $\theta^*$. Under a stricter version of Assumption \ref{Ass:DualVarBounded} and assuming the cost functions are Lipschitz in $\theta$, we demonstrate in the supplementary materials that $\Delta_{Learn}$ is $O(\mathbb{E}\left[\sum_{t=1}^{\tau_A} \lVert \theta^t - \theta^*\rVert_{\theta} \right])$.


Let us now return to Assumption \ref{Ass:DualVarBounded} and present a sufficient condition for this assumption that depends only on the structural properties of the problem and not directly on the iterations of the algorithm.
The type of sufficient condition we consider is an extended Slater condition that requires both lower and upper bound cost constraints to be satisfied in expectation with positive slack for all $\theta \in \Theta$. Let us first define precisely what the average slack is for a given $\theta \in \Theta$.

\begin{definition}\label{Def:SlackTheta}
For a given $\theta \in \Theta$, we define its slack $\delta_\theta \in \bbR$ as $\delta_{\theta} := \mathbb{E}_{\calP} [\max_{z \in\calZ} \, \mathrm{res}(z; \theta,w)]$ with $\mathrm{res}(z; \theta,w) := \min \{ \lVert Tb_k - c_k(z;\theta,w) \rVert_\infty, \lVert c_k(z;\theta,w) - T\alpha_k b_k \rVert_\infty \}$ for all $(z, w) \in \calZ \times \calW$.
\end{definition}

The following proposition uses the average slack to upper bound $C_h$ in Assumption \ref{Ass:DualVarBounded}.

\begin{proposition}\label{Prop:DualQuantityBounded}
Assume that there exists $\delta >0$ such that $\delta_{\theta} \ge \delta$ for all $\theta \in \Theta$, and let $C^{\rhd} :=  2 (\eta\frac{(\bar{C}^2 + \bar{b}^2)}{\sigma_1} +  \bar{f})$. 
Suppose that we use the Euclidean reference function $h(\cdot) := \tfrac{1}{2}\|\cdot\|_2^2$, which corresponds to the traditional projected stochastic subgradient method.
Then, it holds that $C_h \le \max\{ \lVert \lambda^1 \rVert_{\infty}, \sqrt{2} \sqrt{ 0.5 (C^{\rhd}/\delta)^2 + \eta C^{\rhd}} \}$.
\end{proposition}

\subsection{Proof Sketch and Cost Feasibility}
The proof sketch for Theorem \ref{Thm:Master} is informative of how the algorithm works and therefore we outline it here. At a high level the proof consists of two major steps. First, we prove that the $\bbE[\tau_A]$ is close to $T$ for the pure online optimization case. In the general case additional terms depending on how $\theta^t$ converges to $\theta^\ast$ appear. Second, we bound the expected regret up to period $\tau_A$. In particular, we prove $\mathbb{E}[\tau_A D(\sum_{t=1}^{\tau_A} \tfrac{1}{\tau_A} \lambda^t;\theta^*) -\sum_{t=1}^{\tau_A} f(z^t;\theta^*,w^t)]$ upper bounds the regret and is $O(\sqrt{T})$ in the pure online optimization case. Finally, the expected regret up to period $T$ is bounded by the sum of the expected regret up to period $\tau_A$ plus the trivial bound $\bar{f} \mathbb{E}[T-\tau_A]$. 
(Note that the two major steps of our proof mimic those of \citet{balseiro2020dual} but the generality of our setting as well as the presence of parameter learning leads to new complications.)


A key element of the proof is that if we violate the upper cost constraints this occurs near the final period $T$ (as long as we `properly' learn $\theta^*$). 
A solution obtained using Algorithm \ref{Alg:DualMirrorLearning} can not overspend, but may underspend. Proposition \ref{Pro:LowerBoundsSatisfied} shows that the amount of underspending can again be bounded by the sum of terms that arise from the decision-making algorithm (mirror descent) and terms that depend on the convergence of the learning process.
In the pure online optimization case, these lower constraint violations are bounded by $O(\sqrt{T})$ if we use $\eta = \gamma/\sqrt{T}$ with $\gamma>0$ arbitrary. To put this result in context, even if constraint violations can occur their growth is considerably smaller than $T$, which is the rate at which the scale of the constraints in \eqref{Eq:OrigProblem} grow. In the general case, terms depending on how $\theta^t$ converges to $\theta^\ast$ again appear, analogously to Theorem \ref{Thm:Master}. 


\begin{proposition}\label{Pro:LowerBoundsSatisfied}
Assume we run Algorithm \ref{Alg:DualMirrorLearning} under Assumption  \ref{Ass:DualVarBounded} using $\eta_t = \eta$ for all $t \ge 1$. For any $k \in [K]$ with $\alpha_k \ne -\infty$ it holds:
\begin{align*}
     T \alpha_k b_k - \mathbb{E}[\sum_{t=1}^{\tau_A} c_k(z^t;\theta^*,w^t)]  \le &  \left( \frac{\lVert \nabla h(\lambda^1) \rVert_\infty + C_h}{\eta } \right) \frac{\underline{b} + \alpha_k b_k}{\underline{b}} +  \frac{\alpha_k b_k \bar{C}}{ \underline{b}} \\
    &+ \frac{\alpha_k b_k \lVert \mathbb{E}[ \sum_{t=1}^{\tau_{A}} c(z^t;\theta^*,w^t) - c(z^t;\theta^t,w^t) ] \rVert_\infty  }{\underline{b}} \\
    & +\mathbb{E}[\sum_{t=1}^{\tau_A} c_k(z^t;\theta^t,w^t) - c_k(z^t;\theta^*,w^t)].
\end{align*}
\end{proposition}



\section{Experiments\label{Sec:ExsAndExps}}

This section describes the two experiments performed. The first models the problem of a centralized bidder entity bidding on behalf of several clients. Each client has both lower and upper bounds on their desired spending. This experiment uses data from the online advertising company Criteo \citep{diemert2017attribution}. The results show that our methodology spends the clients budgets (mostly) in their desired range, depleting their budgets close to the last period ($T$), and obtaining a higher profit than a highly used heuristic. The second experiment is a linear contextual bandits problem with lower and upper bounds on the number of actions that can be taken. This experiment is illustrative of how different schemes to learn $\theta^*$, \textit{i.e.}, implementations of Step 1. of Algorithm \ref{Alg:DualMirrorLearning}, may be more or less effective depending on the inherent randomness of the data arrivals.  

\subsection{Centralized repeated bidding with budgets}
Consider a centralized bidding entity, which we here call the bidder, who bids on behalf of $K \ge 1$ clients. The bidder can participate in at most $T \ge 1$ auctions which are assumed to use a second-price mechanism. In the case of winning an auction, the bidder can only assign the reward of the auction to at most one client at a time. At the beginning of each auction, the bidder observes a vector $w \in \calW$ of features and a vector $r(w) \in \calR_+^{K}$. Each coordinate of $r(w)$ represents the monetary amount the $k^{th}$ client offers the bidder for the auction reward. For each auction $t \in [T]$, call `$\mathrm{mp}^t$'  to the highest bid from the other bidders. The goal of the bidder is to maximize its profit while satisfying its clients lower and upper spending bounds. Defining $\calX := \{ x \in \bbR_+^K:$ $\sum_{i=1}^K x_i \le 1 \}$, the problem the bidder would like to solve is (special case of Problem \eqref{Eq:OrigProblem}): 
\begin{align}
     & \underset{(z^t,x^t) \in \calR_+ \times \calX: t \in [T]}{\max} \ \ \sum_{t = 1}^T \sum_{k=1}^K (r_k(w^t) - \mathrm{mp}^t) x_k^t \mathds{1} ( z^t \ge \mathrm{mp}^t ) \nonumber \\
    & \text{s.t.} \ \ T \alpha \odot b \le \sum_{t = 1}^T r(w^t)\odot  x^t \mathds{1} \{ z^t \ge \mathrm{mp}^t \}  \le T b .  \nonumber
\end{align}
where $Tb$ represent the maximum the clients would like to spent, and $\alpha \in [0,1)^K$ the minimum percentage to be spent. The pair $(z^t, x^t) \in \bbR_+ \times \Delta$ represents the submitted bid and the probabilistic allocation of the reward chosen by the bidder at period $t$ (we later show that our algorithm uses a binary allocation policy). We use $\mathds{1} \{ z^t \ge \mathrm{mp}^t \}$ to indicate that the bidder wins the auction $t \in [T]$
only if its bid is higher than $\mathrm{mp}^t$. Here we assume $r(\cdot): \calW \rightarrow \bbR_+^K$ as known, but the extension to the case when we need to learn it is natural.

An important property of this problem is that we can implement our methodology without learning the distribution of $\mathrm{mp}$ , making this experiment fall in the pure online optimization case. The latter occurs as $\varphi(\lambda;(w,\mathrm{mp}))$ $=$ $\underset{(z,x) \in \calR_+ \times \calX}{\max} \sum_{k =1}^K (r_k(w)(1-\lambda_k) - \mathrm{mp})x_k \mathds{1} \{ z \ge \mathrm{mp} \}$ can be solved as Algorithm \ref{Alg:SecPrice} shows. 
\floatname{algorithm}{Algorithm} 
\begin{algorithm} 
\caption{Solving $\varphi(\cdot;\cdot,\cdot)$} \label{Alg:SecPrice}
\begin{algorithmic}
\STATE {\textbf{Input:}} Pair $(\lambda, w) \in \calR^{K} \times \calW$, and reward vector $r(w)$.
\STATE 1. Select $k^* \in \arg\underset{k \in [K]}{\max}$ $r_k(w)(1-\lambda_k)$.
\STATE 2. If $r_{k^*}(w)(1-\lambda_{k^*}) \ge 0$ set $z = r_{k^*}(w)(1-\lambda_{k^*})$, $x_{k^*} = 1$ and $x_k = 0$ for all $k \in [K] \ne k^*$, otherwise choose $z = x_k = 0$ for all $k \in [K]$.
\STATE {\textbf{Output:}} $(z,x)$ optimal solution for $\varphi(\lambda;(w,\mathrm{mp}))$.
\end{algorithmic}
\end{algorithm}

\textbf{Experiment Details.} 
This experiment is based on data from Criteo \citep{diemert2017attribution}. Criteo is a Demand-Side Platform (DSP), which are entities who bid on behalf of hundreds or thousands of advertisers which set campaigns with them. The dataset contains millions of bidding logs during one month of Criteo's operation. In all these logs, Criteo successfully acquired ad-space for its clients through real-time second-price auctions (each log represents a different auction and ad-space). Each log contains information about the ad-space and user to which it was shown, the advertiser who created the ad, the price paid by Criteo for the ad-space, and if a conversion occurred or not (besides from other unused columns). The logs from the first three weeks were used as training data, the next two days as validation, and the last week as test.

The experiment was performed as follows. The user's information and advertiser ids from the train data were used to train the neural network for conversion prediction from \citet{pan2018field}. This prediction model was validated using the validation data. Once selected and saved the set of parameters with highest validation AUC, we use the predictions coming from this architecture as if they were the truthful probabilities of conversion. From the test data, we obtained total budgets to spend for each advertiser, assuming that all advertisers expect their budget to be spent at least by 95\% ($\alpha_k = 0.95$ for all $k \in [K]$). To simulate a real operation, we read the test logs in order using batches of 128 logs (as updating a system at every arrival is not realistic). We use 100 simulations for statistical significance and use traditional subgradient descent on Step 7. of Algorithm \ref{Alg:DualMirrorLearning} (more experimental details in the supplement).


\begin{figure}
\center
    \includegraphics[width=0.8\textwidth]{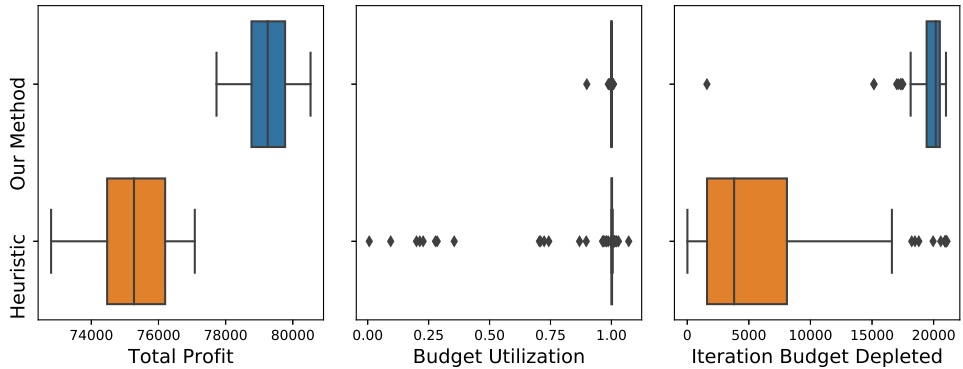}
    \caption{Box plots of the total profit obtained, and average budget utilization and  budget depletion iteration per advertiser over 100 simulations.  Budget utilization corresponds to the percentage of the total budget that an advertiser spent. If an advertiser never depleted its budget, its depletion time equals the simulation length.\label{Fig:ProfitBudgetDep}}
\end{figure}

Figure \ref{Fig:ProfitBudgetDep} shows that our methodology obtains a higher profit in comparison to the baseline. Also, almost all advertisers got their total spending on the feasible range (above 95\% of their total target budget). In addition, advertisers tend to deplete their budgets close to the end of the simulations. Observe that few advertisers spent their budgets in average closer to the beginning rather than the simulations end. We found that those advertisers had relatively small budgets. We saw that as budgets increased, advertisers average depletion time steadily approached the simulation end. 

\subsection{Linear contextual bandits with bounds over the number of actions.}\label{Subsec:LinearContextualBandits}

At each period $t \in [T]$, an agent observes a matrix $W^t \in \bbR^d \times \bbR^n $ and can decide between playing an action or not. If it plays an action, it incurs a cost of $\rho$ and selects a coordinate $i(t) \in [d]$. It then observes a reward $r^t$ with mean $\mathbb{E}[r^t] = (W_{i(t)}^t)^T \theta^* $, where $W_{i(t)}^t$ is the ${i(t)}^{th}$ row of $W^t$ and $\theta^*$ is an unknown parameter. We assume that $r^t = (W_{i(t)}^t)^T \theta^* + \epsilon$ with $\epsilon$ being a zero-mean noise independent of the algorithm history. If the agent does not play an action it incurs no cost. The agent operates at most for $T$ periods, requiring its total cost to be lower than $T$ and higher than $0.5 T$. The agent does not know the distribution $\calW$ over which $W^t$ is sampled (but knows that they are sampled i.i.d.). We can model this problem as having $\calZ = \{ z \in \bbR_+^K: \sum_{i=1}^T z_i \le 1 \}$, $\calW$ being the set of possible matrix arrivals, $f(z;\theta,W^t) = ((W_1^t)^T \theta, \dots, (W_d^t)^T  \theta)^T z$,  and $c(z;\theta,W^t) = (\rho, \dots, \rho) \odot z$. Even when $\calZ$ allows probabilistic allocations, there is always a solution of Step 3. of Algorithm \ref{Alg:DualMirrorLearning} which takes at most one action per period.

\textbf{Experiment Details.} We tried eight combinations of $d \times n$, run Algorithm \ref{Alg:DualMirrorLearning} using  $T = 1000,$ $5000,$ $10000$, use $\rho = 4$, and run 100 simulations of each experiment setting. Each simulation uses a unique seed to create $\theta^*$ and the mean matrix $W$ by sampling i.i.d. $\mathrm{Uniform}(-0.5, 0.5)$ random variables. Both $\theta^*$ and $W$ are then normalized to satisfy $\lVert \theta^* \rVert_2 =1$ and $\lVert W_{d'} \rVert_2 =1$ for all $d' \in [d]$. 

Besides the eight $d \times n$ configurations and three possible $T$ values, we tried six ways of obtaining the revenue terms (making a total of 144 experiment configurations). First, to create $W^t$ we either use $W^t = W$ for all $t \in [T]$, \textit{i.e.} no randomness, or $W^t  = W + \xi^t$ with $\xi^t$ a random matrix with each element being sampled i.i.d. from a $\mathrm{Uniform}(-0.1,0.1)$ random variable. Also, given a selected action $i(t) \in [d]$ on period $t \in [T]$, the observed revenue is either $W_{i(t)}^T \theta^*$ or $W_{i(t)}^T \theta^*$ plus either a $\mathrm{Uniform}(-0.1,0.1)$ or $\mathrm{Uniform}(-0.5,0.5)$ random term. We run Step 7. of algorithm \ref{Alg:DualMirrorLearning} using subgradient descent.

We implemented Step 1. of Algorithm \ref{Alg:DualMirrorLearning} in the following ways. 1. Gaussian Thompson-Sampling as in \citet{agrawal2013thompson}. 2. Least-squares estimation. 3. Ridge regression estimation. 4. Ridge regression estimation plus a decaying randomized perturbation. 5. `Known $\theta^\ast$'. The last method represents the case of a pure online optimization problem. We also solve \eqref{Eq:OrigProblem} optimally for each combination of experiment setting and simulation. In this case $\mathrm{OPT}(\calP) = \mathrm{OPT}(\calP,0)$, and each optimization problem inside $\mathrm{OPT}(\calP,0)$ is a bag problem. Please refer to the supplement for detailed descriptions of the methods, more experimental details, and the proof that $\mathrm{OPT}(\calP) = \mathrm{OPT}(\calP,0)$.

 Table \ref{Tab:50_50} shows the percentage of the average revenue obtained against the best possible revenue achievable over the 100 simulations when using $(d \times n)$ equal to $(50,50)$. A column label, such as $(0.5, 0.1)$ indicates that a $\mathrm{Uniform}(-0.5,0.5)$ is added to the observed revenue and that i.i.d. $\mathrm{Uniform}(-0.1,0.1)$ elements were added element-wise to $W^t$ for each $t \in [T]$. `$0.0$' indicates that no randomness was added either to the revenue or $W^t$ matrices depending on the case. (When $W$ has no randomness, the `Known $\theta^*$' method matches $\mathrm{OPT}(\calP)$ as the optimal action is always the same.)  
\begin{table}[h]
\center
    \begin{tabular}{| p{5.0 cm} | p{1.35 cm} | p{1.35 cm} | p{1.35 cm} | p{1.35 cm} | p{1.35 cm} | p{1.35 cm} |}
    \hline 
    T = 10000, (d $\times$ n) = (50,50) & (0.0,0.0) & (0.1,0.0) & (0.5,0.0) & (0.0,0.1) & (0.1,0.1) & (0.5,0.1) \\ \hline 
    Least Squares &  43.2 & 51.2 & 59.5 & 91.4 & 91.5 & 85.8 \\ \hline
    Thompson Sampling &  98.1 & 13.2 & 2.3 & 93.1 & 19.7 & 3.5 \\ \hline
    Ridge Reg. &  44.9 & 52.9 & 65.0 & 95.6 & 94.5 & 84.9 \\ \hline
    Ridge Reg. + Perturbation &  59.3 & 63.2 & 67.7 & 95.5 & 94.4 & 85.2 \\ \hline
    Known $\theta^*$ &  100 & 100 & 99.9 & 96.7 & 96.7 & 96.8 \\ \hline
    \end{tabular}
    \caption{The results shown are the average revenue over 100 simulations relative to the best value possible. A column label, such as $(0.5, 0.1)$ indicates that a $\mathrm{Uniform}(-0.5,0.5)$ is added to the observed revenue and that $i.i.d.$ $\mathrm{Uniform}(-0.1,0.1)$ elements were added to each coordinate of $W^t$ for each $t \in [T]$. \label{Tab:50_50} }
\end{table}

Table \ref{Tab:50_50} shows interesting patterns. First, Thompson Sampling implemented as in \citet{agrawal2013thompson} was the best performing `learning' method when no randomness was added, but performs terribly when the revenue had added randomness. Differently, the Least Squares and the Ridge Regression methods increased their relative performance greatly when randomness was added to the revenue term. Interestingly, adding uncertainty to ridge regression was a clear improvement when $W^t = W$, but it did not help when $W^t$ had randomness. These results show that how to apply Step 1. of Algorithm \ref{Alg:DualMirrorLearning}  should depend on the application and randomness. Finally, the results shown in Table \ref{Tab:50_50} should be considered just as illustrative as the methods' parameters were not tuned carefully, and neither the method's particular implementation as in the case of Thompson Sampling.  

\bibliographystyle{apalike}
\bibliography{references}

\begin{thebibliography}{}

\bibitem[Agrawal and Devanur, 2016]{agrawal2016linear}
Agrawal, S. and Devanur, N. (2016).
\newblock Linear contextual bandits with knapsacks.
\newblock In {\em Advances in Neural Information Processing Systems}, pages
  3450--3458.

\bibitem[Agrawal and Devanur, 2014]{agrawal2014bandits}
Agrawal, S. and Devanur, N.~R. (2014).
\newblock Bandits with concave rewards and convex knapsacks.
\newblock In {\em Proceedings of the fifteenth ACM conference on Economics and
  computation}, pages 989--1006.

\bibitem[Agrawal and Goyal, 2013]{agrawal2013thompson}
Agrawal, S. and Goyal, N. (2013).
\newblock Thompson sampling for contextual bandits with linear payoffs.
\newblock In {\em International Conference on Machine Learning}, pages
  127--135.

\bibitem[Arora et~al., 2012]{arora2012multiplicative}
Arora, S., Hazan, E., and Kale, S. (2012).
\newblock The multiplicative weights update method: a meta-algorithm and
  applications.
\newblock {\em Theory of Computing}, 8(1):121--164.

\bibitem[Badanidiyuru et~al., 2018]{badanidiyuru2018bandits}
Badanidiyuru, A., Kleinberg, R., and Slivkins, A. (2018).
\newblock Bandits with knapsacks.
\newblock {\em Journal of the ACM (JACM)}, 65(3):1--55.

\bibitem[Balseiro et~al., 2020a]{balseiro2020best}
Balseiro, S., Lu, H., and Mirrokni, V. (2020a).
\newblock The best of many worlds: Dual mirror descent for online allocation
  problems.
\newblock {\em arXiv preprint arXiv:2011.10124}.

\bibitem[Balseiro et~al., 2020b]{balseiro2020dual}
Balseiro, S., Lu, H., and Mirrokni, V. (2020b).
\newblock Dual mirror descent for online allocation problems.
\newblock In {\em International Conference on Machine Learning}, pages
  613--628. PMLR.

\bibitem[Beck and Teboulle, 2003]{beck2003mirror}
Beck, A. and Teboulle, M. (2003).
\newblock Mirror descent and nonlinear projected subgradient methods for convex
  optimization.
\newblock {\em Operations Research Letters}, 31(3):167--175.

\bibitem[Bertsimas and Popescu, 2003]{bertsimas2003revenue}
Bertsimas, D. and Popescu, I. (2003).
\newblock Revenue management in a dynamic network environment.
\newblock {\em Transportation science}, 37(3):257--277.

\bibitem[Boyd and Bilegan, 2003]{boyd2003revenue}
Boyd, E.~A. and Bilegan, I.~C. (2003).
\newblock Revenue management and e-commerce.
\newblock {\em Management science}, 49(10):1363--1386.

\bibitem[Chen et~al., 2017]{chen2017online}
Chen, T., Ling, Q., and Giannakis, G.~B. (2017).
\newblock An online convex optimization approach to proactive network resource
  allocation.
\newblock {\em IEEE Transactions on Signal Processing}, 65(24):6350--6364.

\bibitem[Devanur et~al., 2011]{devanur2011near}
Devanur, N.~R., Jain, K., Sivan, B., and Wilkens, C.~A. (2011).
\newblock Near optimal online algorithms and fast approximation algorithms for
  resource allocation problems.
\newblock In {\em Proceedings of the 12th ACM conference on Electronic
  commerce}, pages 29--38.

\bibitem[Diemert et~al., 2017]{diemert2017attribution}
Diemert, E., Meynet, J., Galland, P., and Lefortier, D. (2017).
\newblock Attribution modeling increases efficiency of bidding in display
  advertising.
\newblock {\em arXiv preprint arXiv:1707.06409}.

\bibitem[Hazan, 2019]{hazan2019introduction}
Hazan, E. (2019).
\newblock Introduction to online convex optimization.
\newblock {\em arXiv preprint arXiv:1909.05207}.

\bibitem[Immorlica et~al., 2019]{immorlica2019adversarial}
Immorlica, N., Sankararaman, K.~A., Schapire, R., and Slivkins, A. (2019).
\newblock Adversarial bandits with knapsacks.
\newblock In {\em 2019 IEEE 60th Annual Symposium on Foundations of Computer
  Science (FOCS)}, pages 202--219. IEEE.

\bibitem[Jenatton et~al., 2016]{jenatton2016adaptive}
Jenatton, R., Huang, J., and Archambeau, C. (2016).
\newblock Adaptive algorithms for online convex optimization with long-term
  constraints.
\newblock In {\em International Conference on Machine Learning}, pages
  402--411. PMLR.

\bibitem[Liakopoulos et~al., 2019]{liakopoulos2019cautious}
Liakopoulos, N., Destounis, A., Paschos, G., Spyropoulos, T., and
  Mertikopoulos, P. (2019).
\newblock Cautious regret minimization: Online optimization with long-term
  budget constraints.
\newblock In {\em International Conference on Machine Learning}, pages
  3944--3952.

\bibitem[Mahdavi et~al., 2012]{mahdavi2012trading}
Mahdavi, M., Jin, R., and Yang, T. (2012).
\newblock Trading regret for efficiency: online convex optimization with long
  term constraints.
\newblock {\em The Journal of Machine Learning Research}, 13(1):2503--2528.

\bibitem[Mehta et~al., 2007]{mehta2007adwords}
Mehta, A., Saberi, A., Vazirani, U., and Vazirani, V. (2007).
\newblock Adwords and generalized online matching.
\newblock {\em Journal of the ACM (JACM)}, 54(5):22.

\bibitem[Pan et~al., 2018]{pan2018field}
Pan, J., Xu, J., Ruiz, A.~L., Zhao, W., Pan, S., Sun, Y., and Lu, Q. (2018).
\newblock Field-weighted factorization machines for click-through rate
  prediction in display advertising.
\newblock In {\em Proceedings of the 2018 World Wide Web Conference on World
  Wide Web}, pages 1349--1357. International World Wide Web Conferences
  Steering Committee.

\bibitem[Wei et~al., 2020]{wei2020online}
Wei, X., Yu, H., and Neely, M.~J. (2020).
\newblock Online primal-dual mirror descent under stochastic constraints.
\newblock In {\em Abstracts of the 2020 SIGMETRICS/Performance Joint
  International Conference on Measurement and Modeling of Computer Systems},
  pages 3--4.

\bibitem[Yuan and Lamperski, 2018]{yuan2018online}
Yuan, J. and Lamperski, A. (2018).
\newblock Online convex optimization for cumulative constraints.
\newblock In {\em Advances in Neural Information Processing Systems}, pages
  6137--6146.

\end{thebibliography}

\section{Additional Theoretical Results and Examples}

\subsection{Different Cases for  $\arg \max_{\gamma \in [0,1]}$ $\mathrm{OPT}(\calP,\gamma)$} \label{App:CasesGamma}

Take the case of $T=1$, $\calZ = \{ [0,1] \}$, $\calW = \{ w_1, w_2 \}$ with equal probability of occurring, $b=1$, and $\alpha = 0.5$. Call $\Pi(\cdot \in A)$ to the function that takes the value of $0$ if condition $A$ holds and $-\infty$ otherwise. We show examples in which $\arg\max_{\gamma \in [0,1]}$ $\mathrm{OPT}(\calP,\gamma)$ match the different cases mentioned in the paper. In most of the examples below the upper bound cost constraint hold trivially, reason why we do not ``enforce'' it using  $\Pi(\cdot \le 1)$, with the only exception on the $\gamma = \tfrac{1}{2}$ example.

\textbf{Infinite solutions.} $f(z;\theta^*,w_1) = z$, $c(z;\theta^*,w_1) = z$, $f(z;\theta^*,w_2) = z$, $c(z;\theta^*,w_2) = z$. In this case $\mathbb{E}[f(z;\theta^*,w)] = z$ and $\mathbb{E}[c(z;\theta^*,w)] = z$. Then, for any $\gamma \in [0,1]$ we have
\begin{align*}
    \mathrm{OPT}(\calP,\gamma) = &  \frac{1}{2} \left( \underset{z \in [0,1]}{\max}  \left\lbrace z + \Pi(\tfrac{1}{2} \le z ) \right\rbrace
   +  \underset{z \in [0,1]}{\max}  \left\lbrace z + \Pi(\tfrac{1}{2} \le z ) \right\rbrace \right) 
\end{align*}
The equality comes directly from the definition of $\mathrm{OPT}(\calP,\gamma)$. Is direct to see that $z=1$ maximizes both optimization problems and that $\mathrm{OPT}(\calP) = \mathrm{OPT}(\calP,\gamma)$ for all $\gamma \in [0,1]$.

\textbf{No solution.} $f(z;\theta^*,w_1) = z$, $c(z;\theta^*,w_1) = 0$, $f(z;\theta^*,w_2) = 0$, $c(z;\theta^*,w_2) = 0$. Since the cost terms are always zero, the cost lower bound $0.5$ is never achieved and no feasible solution exist.

\textbf{$\gamma = \tfrac{1}{2}$ as unique solution}. $f(z;\theta^*,w_1) = z$, $c(z;\theta^*,w_1) = 0$, $f(z;\theta^*,w_2) = -z$, $c(z;\theta^*,w_2) = 2z$. In this case $\mathbb{E}[f(z;\theta^*,w)] = 0$ and $\mathbb{E}[c(z;\theta^*,w)] = z$. Then, for any $\gamma \in [0,1]$ we have
\begin{align*}
    \mathrm{OPT}(\calP,\gamma) = &  \frac{1}{2} \left(  \underset{z \in [0,1]}{\max} \left\lbrace (1-\gamma) z  
    + \Pi(\tfrac{1}{2} \le \gamma z ) \right\rbrace +  \underset{z \in [0,1]}{\max} \left\lbrace -(1-\gamma)z +  \Pi(\tfrac{1}{2} \le (2-\gamma) z ) + \Pi((2-\gamma) z \le 1) \right\rbrace \right) \\
    = &  \frac{1}{2} \left(  (1-\gamma) + \Pi(\tfrac{1}{2} \le \gamma) +  \underset{z \in [0,1]}{\max} \left\lbrace -(1-\gamma)z +  \Pi(\tfrac{1}{2} \le (2-\gamma) z ) + \Pi((2-\gamma) z \le 1) \right\rbrace \right) 
\end{align*}
The second equality uses that the first optimization problem has $z=1$ as its unique optimal solution whenever $\gamma \ne 1$ and that $0 = \mathrm{OPT}(\calP,1) < \mathrm{OPT}(\calP,0.5) = \tfrac{1}{6}$. Is direct from the result above that $\mathrm{OPT}(\calP, \gamma ) = -\infty$ for any $\gamma < 0.5$. Then, we have: 
\begin{align*}
    \mathrm{OPT}(\calP) = &  \frac{1}{2} \left(    \underset{z \in [0,1], \gamma \in [0.5,1)}{\max} (1-\gamma)  -(1-\gamma)z +  \Pi(\tfrac{1}{2} \le (2-\gamma) z ) + \Pi((2-\gamma) z \le 1) \right) \\
    = &  \frac{1}{2} \left( \underset{\gamma \in [0.5,1)}{\max} (1-\gamma) - \tfrac{1-\gamma}{2(2-\gamma)}  \right) 
\end{align*}
The first equality uses the definition of $\mathrm{OPT}(\calP)$ and that we have restricted $\gamma $ to be in $ [0.5,1)$. The second equality uses that for any $\gamma \in [0.5,1)$ the unique optimal is $z(\gamma) = \tfrac{1}{2(2-\gamma)}$ as it maximizes the term $-(1-\gamma)z$ by taking the smallest feasible $z$ value that satisfies the cost lower bound. Finally, the function $\xi(\gamma) := (1-\gamma) - \tfrac{1-\gamma}{2(2-\gamma)} $ is differentiable on $\gamma \in [0.5,1]$ and has strictly negative derivative on $\gamma \in [0.5,1]$, which implies $\xi(0.5) > \xi(\gamma)$  for any  $\gamma \in [0.5,1]$, proving that $\gamma = 0.5$ is the unique optimal solution.

\textbf{$\gamma = 0$ as unique solution}. $f(z;\theta^*,w_1) = z^2$, $c(z;\theta^*,w_1) = z$, $f(z;\theta^*,w_2) = -z$, $c(z;\theta^*,w_2) = 1-z$. In this case $\mathbb{E}[f(z;\theta^*,w)] = 0.5 (z^2-z)$ and $\mathbb{E}[c(z;\theta^*,w)] = 0.5$. Then, for any $\gamma \in [0,1]$ we have
\begin{align*}
    \mathrm{OPT}(\calP,\gamma) = &  \frac{1}{2} \bigg(  \underset{z \in [0,1]}{\max} \left\lbrace z^2 (1-\tfrac{\gamma}{2}) - z \tfrac{\gamma}{2}  
    + \Pi(\tfrac{1}{2} \le  (1-\gamma)z + \tfrac{\gamma}{2} ) \right\rbrace \\
    & + \underset{z \in [0,1]}{\max} \left\lbrace \tfrac{\gamma}{2} z^2 -z (1-\tfrac{\gamma}{2})  +  \Pi(\tfrac{1}{2} \le (1-\gamma)(1-z) + \tfrac{\gamma}{2} ) \right\rbrace \bigg) 
\end{align*}
To understand why $\gamma = 0$ is the unique solution let us analyze both maximization problems separately. The expression $\tfrac{\gamma}{2} z^2 -z (1-\tfrac{\gamma}{2})$ in the second maximization problem is non-positive in $(z,\gamma) \in [0,1]^2$ as we can write it as $(\tfrac{\gamma}{2}z^2 - \tfrac{1}{2} z) - z (\tfrac{1}{2}-\tfrac{\gamma}{2})$ where each term is non-positive. Then, an optimal solution for it is  $(z,\gamma) = (0,0)$ which also satisfies the lower cost constraints. Similarly, the expression $z^2 (1-\tfrac{\gamma}{2}) - z \tfrac{\gamma}{2}$ in $(z,\gamma) \in [0,1]^2$ of the first maximization problem has a maximum in $(z,\gamma) = (1,0)$, optimal pair which also satisfies the lower cost constraints.

\textbf{$\gamma = 1$ as unique solution}. $f(z;\theta^*,w_1) = z$, $c(z;\theta^*,w_1) = 0$, $f(z;\theta^*,w_2) = z$, $c(z;\theta^*,w_2) = z$. In this case $\mathbb{E}[f(z;\theta^*,w)] = z$ and $\mathbb{E}[c(z;\theta^*,w)] = 0.5z$. Then, for any $\gamma \in [0,1]$ we have
\begin{align*}
    \mathrm{OPT}(\calP,\gamma) = &  \frac{1}{2} \left(  \underset{z \in [0,1]}{\max} \left\lbrace  z
    + \Pi(\tfrac{1}{2} \le  \tfrac{\gamma}{2}z ) \right\rbrace +  \underset{z \in [0,1]}{\max} \left\lbrace z +  \Pi(\tfrac{1}{2} \le (1-\tfrac{\gamma}{2})z  ) \right\rbrace \right) \\
\end{align*}
The result is direct as $(z,\gamma) =(1,1)$ is the only pair in $[0,1]^2$ which makes the first optimization problem feasible.

\subsection{Bound on $\Delta_{\mathrm{Learn}}$}

Before stating this subsection result, we define an stricter version of Assumption \ref{Ass:DualVarBounded}

\begin{assumption}[(Stricter) Bounded Dual Iterates]\label{Ass:DualVarBoundedStricter}
There is an absolute constant $C_h' > 0$ such that $\lVert \lambda^t \rVert_1 \le C_h'$ for all $t \in [T]$ almost surely.
\end{assumption}

\begin{proposition}
Run Algorithm \ref{Alg:DualMirrorLearning} with a constant ``step-size'' rule $\eta_t \gets \eta$ for all $t \geq 1$ where $\eta > 0$. Suppose that Assumption \ref{Ass:DualVarBoundedStricter} holds and that $c(\cdot;\cdot,\cdot)$ is Lipschitz on its $\theta$ argument, in particular, that it exists $L_c> 0$, such that $\lVert c(z;\theta,w)-c(z;\theta',w) \rVert_\infty \le L_c \lVert \theta - \theta' \rVert_\theta$ for any $(z,w,\theta, \theta') \in \calZ \times \calW \times \Theta \times \Theta$.
Then, for any distribution $\calP$ over $w \in \calW$, it holds that 
$$\Delta_{Learn} \le L_c \left(1 + C_h' \right) \mathbb{E} \left\lbrack \sum_{t=1}^{\tau_A} \lVert \theta^* - \theta^t \rVert_{\theta}   \right\rbrack. $$
\end{proposition}
\begin{proof}
The proof is obtained directly by bounding each term of $\Delta_{Learn}$ separately. First,
\begin{align*}
     \mathbb{E} \left\lbrack \sum_{t=1}^{\tau_A} c(z^t;\theta^*,w^t) - c(z^t;\theta^t,w^t)  \right\rbrack  
     \le &  \mathbb{E} \left\lbrack \sum_{t=1}^{\tau_A} \lVert c(z^t;\theta^*,w^t) - c(z^t;\theta^t,w^t)) \rVert_\infty  \right\rbrack \\
     \le & L_c \mathbb{E} \left\lbrack \sum_{t=1}^{\tau_A}  \lVert \theta^* - \theta^t \rVert_{\theta}   \right\rbrack,
\end{align*}
where we have used above that $c(\cdot;\cdot,\cdot)$ its Lipschitz on its $\theta$ argument. Now, for any pair $x,y$ of real vectors of same dimension it holds $\lvert x^T y \rvert \le \lVert x \rVert_\infty \lVert y \rVert_1$. Using the latter fact and again that $c(\cdot;\cdot,\cdot)$ is Lipschitz on its $\theta$ argument, we have
\begin{align*}
     \mathbb{E} \left\lbrack \sum_{t=1}^{\tau_A} (c(z^t;\theta^*,w^t) - c(z^t;\theta^t,w^t))^T\lambda^t  \right\rbrack 
     \le & \mathbb{E} \left\lbrack \sum_{t=1}^{\tau_A} \lvert (c(z^t;\theta^*,w^t) - c(z^t;\theta^t,w^t))^T\lambda^t  \rvert \right\rbrack \\
     \le & \mathbb{E} \left\lbrack \sum_{t=1}^{\tau_A} \lVert c(z^t;\theta^*,w^t) - c(z^t;\theta^t,w^t)) \rVert_\infty \lVert \lambda^t \rVert_1  \right\rbrack \\
     \le & L_c \mathbb{E} \left\lbrack \sum_{t=1}^{\tau_A} \lVert \lambda^t \rVert_1 \lVert \theta^* - \theta^t \rVert_{\theta}   \right\rbrack \\
     \le &  L_c C_h' \mathbb{E} \left\lbrack \sum_{t=1}^{\tau_A} \lVert \theta^* - \theta^t \rVert_{\theta}   \right\rbrack.
\end{align*}
\end{proof}

\subsection{Proof That $\mathrm{OPT}(\calP) = \mathrm{OPT}(\calP,0)$ in the Linear Contextual Bandits Experiment and Solving it Efficiently.\label{Subsec:OptOfLinearContextualBandits}}

This appendix subsection shows the following three results. 1. That for any $\rho \ge 0.5$ we have $\mathrm{OPT}(\calP,\gamma) > -\infty$ for all $\gamma \in [0,1]$. 2. That $\mathrm{OPT}(\calP,\gamma) \le \mathrm{OPT}(\calP,0)$ for all $\gamma \in (0,1]$.  3. How to  efficiently solve $\mathrm{OPT}(\calP,0)$. Take $\calZ =\{ z \in \bbR_+^K: \sum_{i=1}^K z_i \le 1 \}$ and $\gamma \in [0,1]$ arbitrary. As notation, here we use superscripts to denote time (but also use $\cdot^T$ to denote dot operation between vectors when need), use subscripts to denote row indexes, and use $W$, $W'$, $W^t$, ${W'}^t$ to represent matrices of size $d \times n$. Also, to shorten notation, we write $\textbf{W}$ to define a sequence $ \{ W^1, \dots, W^T \}$ of $W^t$ matrices (analogous for $\textbf{W}'$). The traditional multiplication between a matrix $A$  of size $d \times n$ and a vector $x$ of size $n$ is written as $A x = ((A_1)^T x, \dots, (A_d)^T x)$. The term inside the outer expectation of $\mathrm{OPT}(\calP,\gamma)$ corresponds to (for $\gamma =1$ the outer expectation can be removed)
\begin{align*}
    O(\bW, \gamma): = \ \ & \underset{z^t \in \calZ: t \in [T]}{\max} \ \ (1-\gamma)\sum_{t = 1}^T (W^t \theta^*)^T z^t + \gamma \mathbb{E}_{W' \sim \calP}[(W' \theta^*)^T z^t] \nonumber \\
    & \text{s.t.} \ \ 0.5 * T \le \rho \sum_{t=1}^T \sum_{i=1}^d z_i^t  \le T.
\end{align*}
Notice that a solution $\bz = \{ z^1,\dots, z^T\}$ is either feasible or infeasible independently of the context vector arrivals $\bW = \{W^1,\dots,W^T\}$ and $\gamma$. For any $\rho \ge 0.5$ and $\gamma \in [0,1]$, it holds $\mathrm{OPT}(\calP, \gamma) > -\infty$ as we can choose $\bz$ satisfying $\sum_{i=1}^d z_i^t =0.5/\rho$ for all $t \in [T]$ (our problem setup uses $\rho =4$). A direct application of Jensen inequality shows $\mathrm{OPT}(\calP, 1) \le \mathrm{OPT}(\calP, 0)$, so let us take $\gamma \in (0,1)$ arbitrary. For any sequence $\bW$, let $\bz_{\gamma}(\bW)$ be an optimal solution of $O(\bW, \gamma)$, we have
\begin{align*}
    \mathrm{OPT}(\calP,\gamma) & = \mathbb{E}_{\bW \sim \calP^T} \left[ (1-\gamma) \sum_{t=1}^T (W^t \theta^*)^T z_{\gamma}^t(\bW) + \gamma \mathbb{E}_{W' \sim \calP} \left[ ({W'} \theta^*)^T z_{\gamma}^t(\bW)  \right]\right] \\
    & = \mathbb{E}_{\bW \sim \calP^T}  \left[ (1-\gamma) \sum_{t=1}^T (W^t \theta^*)^T z_{\gamma}^t(\bW)  \right] +  \mathbb{E}_{\bW \sim \calP^T}  \left[ \gamma  \sum_{t=1}^T \mathbb{E}_{W' \sim \calP} [({W'} \theta^*)^T z_{\gamma}^t(\bW)]  \right] \\
    & = \mathbb{E}_{\bW \sim \calP^T}  \left[ (1-\gamma) \sum_{t=1}^T (W^t \theta^*)^T z_{\gamma}^t(\bW)  \right] + \mathbb{E}_{\bW \sim \calP^T} \left[ \mathbb{E}_{\bW' \sim \calP^T}  \left[  \gamma  \sum_{t=1}^T ({W'}^t \theta^*)^T z_{\gamma}^t(\bW)  \right] \right] \\
    & = \mathbb{E}_{\bW \sim \calP^T}  \left[ (1-\gamma) \sum_{t=1}^T (W^t \theta^*)^T z_{\gamma}^t(\bW)  \right] + \mathbb{E}_{\bW \sim \calP^T, \bW' \sim \calP^T}  \left[  \gamma  \sum_{t=1}^T (({W}^t)^T \theta^*)^T z_{\gamma}^t(\bW')  \right] \\
    & = \mathbb{E}_{\bW \sim \calP^T} \left[  \sum_{t=1}^T (W^t \theta^*)^T \left((1-\gamma) z_{\gamma}^t(\bW) + \gamma \mathbb{E}_{\bW' \sim \calP^T}  [z_{\gamma}^t(\bW')] \right) \right] \\
    & \le \mathbb{E}_{\bW \sim \calP^T} \left[  \sum_{t=1}^T (W^t \theta^*)^T z_{0}^t(\bW)  \right] = \mathrm{OPT}(\calP, 0).
\end{align*}
The second equality uses the linearity of the expectation operator, the third uses that each ${W'}^t$ is sampled i.i.d. from $\calP$, the fourth that $\bW$ and $\bW'$ are i.i.d. and can be exchanged, the fifth uses the linearity of the expectation operator again, and the final inequality uses the definition of $\bz_0(\bW)$. In particular, the last inequality uses that $(1-\gamma) \bz_{\gamma}(\bW) + \gamma \mathbb{E}_{\bW' \sim \calP^T}  [\bz_{\gamma}(\bW')]$ is a feasible solution of $O(\bW,0)$. Finally, notice that for any given $\bW$ solving $O(\bW, 0)$ is equivalent to solving the following knapsack problem
\begin{align*}
    O(\bW, 0) &= \max_{y^t \in [0,1]: t \in T} \ \ \sum_{t=1}^T \left( \max_{i \in [d]} \, (W_i^t)^T \theta^* \right) y^t \nonumber \\
    & \text{s.t.} \ \ 0.5 * T \le \rho \sum_{t=1}^T y^t  \le T.
\end{align*}
Let $\{ m_1, \dots, m_T \}$ represent the sequence $\{ \max_{i \in [d]} \, (W_i^t)^T \theta^* \}_{t=1}^T$ ordered from biggest to smallest value. Then, is not hard to see that
\begin{align*}
    O(\bW, 0) &= \max_{i_{\max} \in \left[ \left\lceil \tfrac{T}{2\rho} \right\rceil, \left\lfloor \tfrac{T}{\rho} \right\rfloor \right]} \ \ \sum_{i=1}^{i_{max}} m_i, 
\end{align*}
where $\lceil \cdot \rceil$ and $\lfloor \cdot \rfloor$ are the traditional ceiling and floor integer functions respectively.

\section{Proofs \label{Chap:SupOnline_2}}

\subsection{Proof of Proposition \ref{Pro:WeakDuality}}
\begin{proof}
Let $\calP^T$ be the distribution from which the $(w^1,\dots,w^T)$ vectors are sampled, with each $w^t$ being sampled $i.i.d.$ from $\calP$. For any $\gamma \in [0,1]$
\begin{align*}
    & \mathrm{OPT}(\calP, \gamma)\\
    & = \mathbb{E}_{\calP^T}
    \left\lbrack
    \begin{array}{c l}	
         \underset{z^t \in \calZ: t \in [T]}{\max} \ \ \sum_{t=1}^T (1-\gamma) f(z^t;\theta^*,w^t) +\gamma \mathbb{E}_{\calP}[f(z^t;\theta^*,w)]    \\
         \textrm{s.t.} \ \ T \alpha_k b_k \le \sum_{t=1}^T (1-\gamma) c_k(z^t;\theta^*,w^t) + \gamma \mathbb{E}_{\calP}[c_k(z^t;\theta^*,w)]  \le   T b_k \textrm{ for all } k \in {[K]} 
    \end{array}
    \right\rbrack \\
    & \le \mathbb{E}_{\calP^T} \Bigg\lbrack \underset{z^t \in \calZ: t\in [T]}{\max} \Bigg\lbrace \sum_{t=1}^T  (1-\gamma) \left( f(z^t;\theta^*,w^t) - \lambda^T c(z^t;\theta^*,w^t) \right)  + \gamma \mathbb{E}_\calP [ f(z^t;\theta^*,w) - \lambda^T c(z^t;\theta^*,w)] \Bigg\rbrace + T p(\lambda) \Bigg\rbrack \\
    & = \mathbb{E}_{\calP^T} \left\lbrack \sum_{t=1}^T \underset{z^t \in \calZ: t \in {T}}{\max}   (1-\gamma) \left( f(z^t;\theta^*,w^t) - \lambda^T c(z^t;\theta^*,w^t) \right) + \gamma \mathbb{E}_\calP [ f(z^t;\theta^*,w) - \lambda^T c(z^t;\theta^*,w)]   \right\rbrack   + T p(\lambda) \\
    & \le (1-\gamma) \mathbb{E}_{\calP^T} \left\lbrack \sum_{t=1}^T \underset{z^t \in \calZ: t \in {T}}{\max}  \,   f(z^t;\theta^*,w^t) - \lambda^T c(z^t;\theta^*,w^t)  \right\rbrack \\
    & + \gamma \mathbb{E}_{\calP^T} \left\lbrack \sum_{t=1}^T \underset{z^t \in \calZ: t \in {T}}{\max} \mathbb{E}_\calP [ f(z^t;\theta^*,w) - \lambda^T c(z^t;\theta^*,w)]   \right\rbrack + T p(\lambda) \\
    & \le (1-\gamma) T \mathbb{E}_{\calP} \left\lbrack  \underset{z \in \calZ}{\max}  \, f(z;\theta^*,w) - \lambda^T c(z;\theta^*,w)  \right\rbrack  + \gamma T \, \underset{z \in \calZ}{\max} \, \mathbb{E}_\calP \left\lbrack f(z;\theta^*,w) - \lambda^T c(z;\theta^*,w) \right\rbrack + T p(\lambda)  \\
    & \le (1-\gamma) T \mathbb{E}_{\calP} \left\lbrack  \varphi(\lambda;\theta^*,w) \right\rbrack 
     + \gamma T  \mathbb{E}_\calP \left\lbrack \underset{z \in \calZ}{\max}  \, f(z;\theta^*,w) - \lambda^T c(z;\theta^*,w) \right\rbrack  + T p(\lambda) \\
     & = T \mathbb{E}_{\calP} \left\lbrack   \varphi(\lambda;\theta^*,w) \right\rbrack + T p(\lambda) \\
     & = T D(\lambda;\theta^*)
\end{align*}
The first equality is the definition of $\mathrm{OPT}(\calP, \gamma)$, the first inequality uses Lagrangian duality for both the lower and upper bounds constraints, the second equality uses that $p(\lambda)$ can be moved outside the expectation and that the sum can be changed with the maximization operator as there is no constraint linking the $z^t$ variables. The second inequality uses that for any $a(\cdot)$ and $b(\cdot)$ real valued functions we have  $max_{z \in \calZ}$ $\{a(z) + b(z)\}$ $\le $ $max_{z \in \calZ}$ $a(z)$ $+$ $max_{z \in \calZ}$ $b(z)$, the third inequality uses that all $w^t$ are $i.i.d$ sampled, that all maximization problems are the same in the first term, and that the outer expectation can be removed from the second term. The fourth inequality uses the definition of $\varphi(\cdot;\cdot,\cdot)$ and that $\max_{z \in \calZ} \, \mathbb{E}_{\calP}[\cdot] \le \mathbb{E}_{\calP}[ \max_{z \in \calZ} \, \cdot]$. Finally, we use the definition of $\varphi(\cdot;\cdot,\cdot)$ again and the fact that $\gamma + (1-\gamma) =1$.
\end{proof}

\subsection{Proof of Proposition \ref{Pro:StochSubg}}
\begin{proof}
First note that the $p(\cdot)$ function used inside $D(\cdot;\cdot)$ is convex since $b \geq 0$ and $\alpha \in [-1, 1)^K$. We need to prove that
$D(\lambda;\theta) +\mathbb{E}_\calP[\tilde{g}(\lambda;\theta,w)]^T(\lambda' -\lambda) \le D(\lambda';\theta)$ for any $\lambda \in \Lambda$ and $\lambda' \in \Lambda$. Let $p'$ be any member of $ \partial p(\lambda)$, we have
\begin{align*}
     D(\lambda;\theta) +\mathbb{E}_\calP[\tilde{g}(\lambda;\theta,w)]^T(\lambda' -\lambda) 
    = & \mathbb{E}_{\calP} [\varphi(\lambda;\theta,w) +   p(\lambda) +  \tilde{g}(\lambda;\theta,w)^T(\lambda'-\lambda)] \\
    = & \mathbb{E}_{\calP} [f(z(\lambda;\theta,w);\theta,w) - (\lambda')^T c(z(\lambda;\theta,w);\theta,w) + p(\lambda) + {p'}^T (\lambda' - \lambda)] \\
    \le & \mathbb{E}_{\calP} [f(z(\lambda;\theta,w);\theta,w) - (\lambda')^T c(z(\lambda;\theta,w);\theta,w) + p(\lambda')] \\
    \le & D(\lambda';\theta).
\end{align*}
The first equality uses the definition of $D(\lambda;\theta)$, the second equality uses the definition of $\tilde{g}(\lambda;\theta,w)$, the first inequality uses the subgradient inequlity for $p(\cdot)$, and the second inequality uses the definition of $D(\lambda';\theta)$.
\end{proof}

\subsection{Intermediate Results}

The following propositions were not mentioned in the paper. Proposition \ref{Pro:KKT_for_h} shows an inequality that holds for Step 7. of Algorithm \ref{Alg:DualMirrorLearning} under the conditions given for $\Lambda$ and $h(\cdot)$ on the paper. Propositions \ref{Pro:StopTimeBound} and \ref{Pro:CorrectRegretUpToTau_A} are intermediate steps to prove Theorem \ref{Thm:Master}. Proposition \ref{Pro:StopTimeBound} bounds  $T- \tau_A$ in expectation. Proposition \ref{Pro:CorrectRegretUpToTau_A} shows an upper bound for the regret that Algorithm \ref{Alg:DualMirrorLearning} up to period $\tau_A$. Proposition \ref{Pro:KeySlater} is the key result needed to prove Proposition \ref{Prop:DualQuantityBounded}.

\begin{proposition}\label{Pro:KKT_for_h}
Let $\Lambda \subseteq \bbR^K$ be a set which can be defined separately for each dimension $k \in [K]$,  either being $\Lambda_k = \bbR$ or $\Lambda_k = \bbR_+$. Let $h(\cdot): \Lambda \rightarrow \bbR$ be a function that satisfies $h(\lambda) = \sum_{k = 1}^K$ $h_k(\lambda_k)$, with $h_k(\cdot)$ being a strongly convex univariate differentiable function for all $k \in [K]$. Given arbitrary $\lambda' \in \Lambda$, $\tilde{g} \in \bbR^K$, and $\eta >0$ define $ \lambda^+ = \arg \min_{\lambda \in \Lambda} \, \lambda^T \tilde{g}^t + \tfrac{1}{\eta} V_h(\lambda, \lambda')$. Then, for all $k \in [K]$ it holds
\begin{enumerate}
    \item If $\Lambda_k = \bbR$, then $\dot{h}_k(\lambda_k^+) = \dot{h}_k(\lambda_k') - \eta \tilde{g}_k$.
    \item If $\Lambda_k = \bbR_+$, then $\dot{h}_k(\lambda_k^+) = \dot{h}_k(\lambda_k') - \eta \tilde{g}_k$ if $\lambda_k^+>0$ or $\dot{h}_k(\lambda_k^+) \ge \dot{h}_k(\lambda_k') - \eta \tilde{g}_k$ if $\lambda_k^+=0$.
\end{enumerate}
Therefore, proving that $\nabla h(\lambda^+) \ge \nabla h(\lambda') - \eta \tilde{g}$.
\end{proposition}
\begin{proof}
Notice that $\min_{\lambda \in \Lambda} \, \lambda^T \tilde{g}^t + \tfrac{1}{\eta} V_h(\lambda, \lambda')$ $=$ $\sum_{k \in [K]}$ $\min_{\lambda_k \in \Lambda_k} \phi_k(\lambda_k;\lambda_k',\tilde{g}_k)$ with $\phi_k(\lambda_k;\lambda_k',\tilde{g}_k) := \tilde{g}_k \lambda_k + \tfrac{1}{\eta}(h_k(\lambda_k) - h_k(\lambda_k') - \dot{h}_k(\lambda_k')(\lambda_k - \lambda_k'))$ for all $k \in [K]$.  Then, independently per coordinate we minimize a strongly convex function under a non-empty closed convex set, which shows that $\lambda_k^+$ exists for each $k \in [K]$. Also, $\lambda_k^+$ can be found using first order necessary optimality conditions for each $k \in [K]$. Taking $k \in [K]$ arbitrary, we split the proof in two cases.

$\Lambda_k = \bbR$. By first order optimality conditions we immediately obtain $\dot{h}_k(\lambda_k^+) = \dot{h}(\lambda_k') - \eta \tilde{g}_k$.

$\Lambda_k = \bbR_+$. Define $\Pi_+(\cdot): \bbR \rightarrow \{0\} \cup \{\infty\}$ as the convex function that takes the value of $0$ if its input is non-negative and $\infty$ otherwise. Then, the minimization problem for dimension $k$ can be re-written as $\min_{\lambda_k \in \Lambda_k} \phi_k(\lambda_k;\lambda_k',\tilde{g}_k) + \Pi_+(\lambda_k)$. First order necessary optimality conditions imply $0 \in \partial(\phi_k(\lambda_k^+;\lambda_k',\tilde{g}_k) + \Pi_+(\lambda_k^+)  )$. Then, there exists $y \in \partial (\Pi_+(\lambda_k^+))$, such that $\dot{h}_k(\lambda_k^+) = \dot{h}(\lambda_k') - \eta \tilde{g}_k - \eta y$. The result is obtained directly using that $\partial(\Pi_+(\lambda_k))$ is equal to $\{0\}$ when $\lambda_k >0$ and equal to $\bbR_-$ when $\lambda_k = 0$. 
\end{proof}

\begin{proposition}\label{Pro:StopTimeBound}
Run Algorithm \ref{Alg:DualMirrorLearning} with a constant ``step-size'' rule $\eta_t \gets \eta$ for all $t \geq 1$ where $\eta > 0$. Suppose that Assumption \ref{Ass:DualVarBoundedStricter} holds and take $\tau_A$ as in Definition \ref{Def:StopTimeTauA}. Then, 
\begin{align*}
\mathbb{E} \left[ T - \tau_{A} \right]  & \le \frac{\bar{C}}{\underline{b}}+\frac{C_h + \lVert \nabla h(\lambda^1) \rVert_\infty}{\eta \underline{b}} + \frac{\lVert \mathbb{E}[ \sum_{t=1}^{\tau_{A}} c(z^t;\theta^*,w^t) - c(z^t;\theta^t,w^t) ] \rVert_\infty }{\underline{b}}.
\end{align*}
\end{proposition}
 \begin{proof}
Let $k' \in [K]$ be the index of the first violated upper cost bound, \textit{i.e.} the index which activates the stop time $\tau_A$. Here we assume that some upper cost bound constraint is violated,  \textit{i.e.} that $\tau_A < T$, if not the result is trivial. Step 6. of Algorithm \ref{Alg:DualMirrorLearning} defines $\tilde{g}_{k'}^t  = - c_{k'}(z^t;\theta^t,w^t) + b_{k'} \left(\mathds{1}(\lambda_{k'} \ge 0) + \alpha_{k'} \mathds{1}(\lambda_{k'} < 0) \right)$, which can be upper bounded by  $\tilde{g}_{k'}^t \le -c_{k'}(z^t;\theta^t,w^t) + b_{k'}$. Using the definition of $\tau_A$ and $\tilde{g}_{k'}^t$ we have
\begin{align*}
    \sum_{t=1}^{\tau_A} \tilde{g}_{k'}^t  \le & b_{k'} \tau_A - \sum_{t=1}^{\tau_A}  c_{k'}(z^t;\theta^*,w^t)  + \left(  \sum_{t=1}^{\tau_A}  (c_{k'}(z^t;\theta^*,w^t) - c_{k'}(z^t;\theta^t,w^t))\right)  \\
    \le & b_{k'} \tau_A -b_{k'} T  +  \bar{C}  + \left(  \sum_{t=1}^{\tau_A}  (c_{k'}(z^t;\theta^*,w^t) - c_{k'}(z^t;\theta^t,w^t))\right)  \\
     \Rightarrow  T - \tau_A  \le & \frac{1}{b_{k'}}\left( \bar{C} - \sum_{t=1}^{\tau_A}  \tilde{g}_{k'}^t \right)  + \frac{1}{b_{k'}}\left(  \sum_{t=1}^{\tau_A}  (c_{k'}(z^t;\theta^*,w^t) - c_{k'}(z^t;\theta^t,w^t))\right).
\end{align*}
Using that our update rule satisfies 
$\dot{h}_{k'}(\lambda_{k'}^{t+1}) \ge \dot{h}_{k'}(\lambda_{k'}^t) - \eta \tilde{g}_{k'}^t$ for all $t \le \tau_A$ and the definitions of $\underline{b}$ and $C_h$, we get
\begin{align*}
   - \sum_{t=1}^{\tau_{A}}  \tilde{g}_{k'}^t & \le \frac{1}{\eta} \left( \dot{h}_{k'}(\lambda_{k'}^{\tau_{A}+1}) - \dot{h}_{k'}(\lambda_{k'}^1) \right) \\
   \Rightarrow T - \tau_A  & \le \frac{\bar{C}}{b_{k'}}+\frac{\dot{h}_{k'}(\lambda_{k'}^{\tau_A+1}) - \dot{h}_{k'}(\lambda_{k'}^1)}{\eta b_{k'}}  + \left( \frac{\sum_{t=1}^{\tau_A} (c_{k'}(z^t;\theta^*,w^t) - c_{k'}(z^t;\theta^t,w^t))}{b_{k'}} \right)    \\
    \Rightarrow \mathbb{E} \left[ T - \tau_{A} \right]  & \le \frac{\bar{C}}{\underline{b}}+\frac{C_h + \lVert \nabla h(\lambda^1) \rVert_\infty}{\eta \underline{b}}  + \left( \frac{ \lVert \mathbb{E}[ \sum_{t=1}^{\tau_A}  c(z^t;\theta^*,w^t) - c(z^t;\theta^t,w^t) ] \rVert_\infty }{\underline{b}} \right)   
\end{align*}
\end{proof}

\begin{proposition}\label{Pro:CorrectRegretUpToTau_A}
Run Algorithm \ref{Alg:DualMirrorLearning} with a constant ``step-size'' rule $\eta_t \gets \eta$ for all $t \geq 1$ where $\eta > 0$. Denote $\bar{\lambda}^{\tau_{A}} = \frac{\sum_{t=1}^{\tau_{A}} \lambda^t}{\tau_{A}}$ ($\tau_{A}$ as in Definition \ref{Def:StopTimeTauA}). It holds
\begin{align*}
    \mathbb{E} \left\lbrack  \tau_{A} D(\bar{\lambda}^{\tau_{A}};\theta^*) - \sum_{t=1}^{\tau_{A}}  f(z^t;\theta^t,w^t)  \right\rbrack  \le & \frac{ 2 (\bar{C}^2 +\bar{b}^2)}{\sigma_1} \eta \mathbb{E}[\tau_{A}] + \frac{1}{\eta} V_h(\lambda,\lambda^1)   \\
    + & \mathbb{E} \left\lbrack \sum_{t=1}^{\tau_A} (c(z^t;\theta^*,w^t) - c(z^t;\theta^t,w^t))^T\lambda^t  \right\rbrack.
\end{align*}
\end{proposition}
\begin{proof}

For clarity, we sometimes use $\mathbb{E}_w[\cdot]$, $\mathbb{E}_{w^t}[\cdot]$, or $\mathbb{E}_{\calH^{t-1}}[\cdot]$ to indicate the random variable over which the expectation is taken. Using $\mathbb{E}[\cdot]$ indicates that the expectation is taken over the ``whole'' randomness of Algorithm \ref{Alg:DualMirrorLearning}.  Call $\tilde{g}^t$ the vector obtained in Step 6. and define  $\mathbb{E}[\tilde{g}^t] = g^t$. The proof is composed of three steps. 1. Bounding $\tilde{g}^t$. 2. Upper bounding $\mathbb{E} \left\lbrack \sum_{s=1}^{\tau_{A}}  (g^{s})^T (\lambda^{s}-\lambda) \right\rbrack$. 3. Lower bounding $\mathbb{E} \left\lbrack \sum_{s=1}^{\tau_{A}}  ( g^{s} )^T ( \lambda^{s}-\lambda ) \rangle \right\rbrack$. The upper and lower bounds match the left and right hand side of the terms in Proposition \ref{Pro:CorrectRegretUpToTau_A}.  

\noindent \textbf{Step 1.}  Upper bound for $\mathbb{E}[\lVert \tilde{g}^t \rVert_\infty^2]$.
\begin{align*}
    \mathbb{E}[\lVert \tilde{g}^t \rVert_\infty^2] \le \mathbb{E}[(\lVert c(z^t;\theta^t,w^t) \rVert_\infty + \lVert b \rVert_\infty)^2] \le  2 \mathbb{E}[\lVert c(z^t;\theta^t,w^t) \rVert_\infty^2 + \lVert b \rVert_\infty^2] \le 2(\bar{C}^2+\bar{b}^2)
\end{align*}

\noindent \textbf{Step 2.} Upper bound for $\mathbb{E} \left\lbrack \sum_{s=1}^{\tau_{A}}  ( g^{s} )^T (\lambda^{s}-\lambda ) \right\rbrack$. Notice
\begin{align}
    &  \mathbb{E}_{w^t}[ ( \tilde{g}^t )^T ( \lambda^t-\lambda ) |\lambda^t,\theta^t] \nonumber \\
    \le & \mathbb{E}_{w^t} \left\lbrack ( \tilde{g}^t )^T ( \lambda^t-\lambda^{t+1}  ) + \frac{1}{\eta} V_h(\lambda,\lambda^t) -  \frac{1}{\eta} V_h(\lambda,\lambda^{t+1}) - \frac{1}{\eta} V_h(\lambda^{t+1},\lambda^{t}) \big| \lambda^t,\theta^t \right\rbrack  \nonumber \\
    \le & \mathbb{E}_{w^t} \left\lbrack ( \tilde{g}^t )^T (\lambda^t - \lambda^{t+1} ) + \frac{1}{\eta} V_h(\lambda,\lambda^t) - \frac{1}{\eta} V_h(\lambda,\lambda^{t+1}) - \frac{\sigma_1}{2\eta} \lVert \lambda^{t+1} -\lambda^t\rVert_1^2 \big| \lambda^t,\theta^t  \right\rbrack \nonumber \\
    \le & \mathbb{E}_{w^t} \left\lbrack \frac{\eta}{\sigma_1}\lVert \tilde{g}^t \rVert_\infty^2 + \frac{1}{\eta}V_h(\lambda,\lambda^t) - \frac{1}{\eta}V_h(\lambda,\lambda^{t+1}) \big| \lambda^t,\theta^t  \right\rbrack \nonumber \\
    \le & \frac{2 \eta}{\sigma_1}(\bar{C}^2 +\bar{b}^2) + \frac{1}{\eta} V_h(\lambda,\lambda^t) -  \mathbb{E}_{w^t} \left\lbrack \frac{1}{\eta}  V_h(\lambda,\lambda^{t+1}) \big| \lambda^t,\theta^t  \right\rbrack, \label{Eq:AppProRegretUpToStable_0}
\end{align}
where the first inequality is due to the three point property (Lemma 4.1 of \citet{beck2003mirror}), the second uses $V_h(\lambda^{t+1},\lambda^t) \ge \tfrac{\sigma_1}{2} \lVert \lambda^{t+1} -\lambda^t \rVert_1^2 $ given that $h(\cdot)$ is $\sigma_1$-strongly convex with respect to the $\lVert \cdot \rVert_1$  norm, the third uses that for any two vectors $a^1$ and $a^2$ of same dimension it holds $(a^1)^T a^2 +0.5 \lVert a^1  \rVert_\infty^2 \ge -0.5 \lVert a^2 \rVert_1^2 $, and the final inequality is just understanding which terms are constant under the conditional expectation. Taking $E_{\calH^{t-1}}[\cdot]$ over both sides of equation \eqref{Eq:AppProRegretUpToStable_0} and using the law of total expectation we get
\begin{align}
    \mathbb{E}[\eta ( g^t )^T (\lambda^t-\lambda ) ] \le  \frac{2(\bar{C}^2 +\bar{b}^2)}{\sigma_1} \eta^2 + \mathbb{E} \left\lbrack  V_h(\lambda,\lambda^t)  \right\rbrack - \mathbb{E} \left\lbrack  V_h(\lambda,\lambda^{t+1})  \right\rbrack, \label{Eq:AppProRegretUpToStable_1}
\end{align}
since the pair $(\lambda^t,\theta^t)$ is completely determined by $\calH^{t-1} \cup \{ w^t \}$ and that $w^t$ is independent of $\calH^{t-1}$.
Then, regardless of the value of $\tau_{A}$, using the telescopic property and that $V_h(\cdot,\cdot)$ is non-negative we obtain
    \begin{align*}
         \mathbb{E}  \left\lbrack \sum_{s=1}^{\tau_{A}} ( g^{s} )^T ( \lambda^{s}-\lambda ) \right\rbrack  \le   \frac{ 2 (\bar{C}^2 +\bar{b}^2)}{\sigma_1} \eta \mathbb{E}[\tau_{A}] + \frac{V_h(\lambda,\lambda^1)}{\eta}.
    \end{align*}
    
\noindent \textbf{Step 3.} Lower bounds for $\mathbb{E} \left\lbrack \sum_{s=1}^{\tau_{A}}  ( g^{s} )^T (\lambda^{s}-\lambda ) \right\rbrack$. By definition of $g^t$, using the subgradient inequality we get
\begin{align*}
   & ( g^t )^T (\lambda^t - \lambda)  \ge   D(\lambda^t;\theta^t) - D(\lambda;\theta^t) 
     \ge   D(\lambda^t;\theta^t) - \left( \mathbb{E}_{w} [\varphi(\lambda;\theta^t,w)] + \sum_{k \in [K]} b_k([\lambda_k]_+ - \alpha_k [-\lambda_k]_+)  \right). 
\end{align*}
For any  $w \in \mathcal{W}$ we have 
$
    f(z(\lambda^t;\theta^t,w);\theta^t,w) - \lambda^T c(z(\lambda^t;\theta^t,w);\theta^t,w)  \le \varphi(\lambda;\theta^t,w)
$
 as by definition $z(\lambda^t;\theta^t,w)$ is an optimal solution of $\varphi(\lambda^t;\theta^t,w)$ not of $\varphi(\lambda;\theta^t,w)$. Defining $\bar{\lambda}^{\tau_A} := \frac{1}{\tau_{A}} \sum_{t=1}^{\tau_{A}} \lambda^t$, taking $\lambda =  (0,0,\dots,0)$, and summing from one to $\tau_A$ we get
 \begin{align}
     & \sum_{t=1}^{\tau_{A}} ( g^t )^T (\lambda^t - 0 ) \nonumber \\
     \ge & \sum_{t=1}^{\tau_{A}}   D(\lambda^t;\theta^t) - \mathbb{E}_{w} [f(z(\lambda^t;\theta^t,w);\theta^t,w)] \nonumber \\
     \ge & \sum_{t=1}^{\tau_{A}}  \left( D(\lambda^t;\theta^*) - \mathbb{E}_{w} [f(z(\lambda^t;\theta^t,w);\theta^*,w)]  \right) + \sum_{t=1}^{\tau_{A}} \left( D(\lambda^t;\theta^t) - D(\lambda^t;\theta^*) \right) \nonumber \\
     & +   \sum_{t=1}^{\tau_{A}} \left( \mathbb{E}_{w} [f(z(\lambda^t;\theta^t,w);\theta^*,w) - f(z(\lambda^t;\theta^t,w);\theta^t,w)]  \right) \nonumber \\
     \ge & \left( \tau_{A} D(\bar{\lambda}^{\tau_{A}};\theta^*) - \sum_{t=1}^{\tau_{A}}   \mathbb{E}_{w} [f(z(\lambda^t;\theta^*,w);\theta^*,w)] \right)  + \sum_{t=1}^{\tau_{A}} \left( D(\lambda^t;\theta^t) - D(\lambda^t;\theta^*) \right) \nonumber \\
     & +   \sum_{t=1}^{\tau_{A}} \left( \mathbb{E}_{w} [f(z(\lambda^t;\theta^t,w);\theta^*,w) - f(z(\lambda^t;\theta^t,w);\theta^t,w)]  \right) .\label{Eq:AppProRegretUpToStable}
\end{align}

Taking expectation over \eqref{Eq:AppProRegretUpToStable} and using the results from Step 2. we get
\begin{align}
   & \mathbb{E} \left\lbrack  \tau_{A} D(\bar{\lambda}^{\tau_{A}};\theta^*) - \sum_{t=1}^{\tau_{A}}  \mathbb{E}_w \left\lbrack  f(z(\lambda^t;\theta^t,w);\theta^*,w) \right\rbrack  \right\rbrack  \le \frac{ 2 (\bar{C}^2 +\bar{b}^2)}{\sigma_1} \eta \mathbb{E}[\tau_{A}] + \frac{1}{\eta} V_h(0,\lambda^1)  \nonumber \\
    + & \mathbb{E} \left\lbrack \sum_{t=1}^{\tau_{A}}  \mathbb{E}_w [ c(z(\lambda^t;\theta^t,w);\theta^t,w) ]^T \lambda^t \right\rbrack - \mathbb{E} \left\lbrack \sum_{t=1}^{\tau_{A}}  \mathbb{E}_w [   c(z(\lambda^t;\theta^t,w);\theta^*,w) ]^T \lambda^t \right\rbrack, \label{Eq:AppProRegretUpToStable_2}
\end{align}
where we have used the definition of $D(\cdot,\cdot)$ to reduce the second line of \eqref{Eq:AppProRegretUpToStable_2} to use only the cost functions. 
Equation \eqref{Eq:AppProRegretUpToStable_2} almost matches the conclusion of Theorem \ref{Thm:Master} except that \eqref{Eq:AppProRegretUpToStable_2}  uses a $\mathbb{E}[\sum_{t=1}^{\tau_A} \mathbb{E}_{w}[\cdot]]$ term, while the theorem uses $\mathbb{E}[\sum_{t=1}^{\tau_A} \cdot ]$. The previous issue is solved using the Optional Stopping Theorem. We prove now that $\mathbb{E} \left\lbrack  \sum_{t=1}^{\tau_{A}}    f(z(\lambda^t;\theta^t,w^t);\theta^*,w^t)  \right\rbrack $ equals  $\mathbb{E} \left\lbrack  \sum_{t=1}^{\tau_{A}} \mathbb{E}_w \left[   f(z(\lambda^t;\theta^t,w);\theta^*,w) \right] \right\rbrack $ (the analysis for the cost terms appearing in the second line of \eqref{Eq:AppProRegretUpToStable_2} is analogous). First notice
 $$\mathbb{E}_{w} \left\lbrack f(z(\lambda;\theta,w);\theta^*,w) | \lambda  = \lambda^t, \theta = \theta^t  \right\rbrack =  \mathbb{E}_{w}\left\lbrack f(z(\lambda^t;\theta^t,w);\theta^*,w) | \calH^{t-1}  \right\rbrack.$$
Define the martingale $M^t = \sum_{s=1}^t f(z(\lambda^{s};\theta^s,w^{s});\theta^*,w^{s})  $ $-$ $  \mathbb{E}_w[f(z(\lambda^{s};\theta^s,w);\theta^*,w)|\calH^{s - 1}] $ for all $t \le T$. Using that $\tau_{A}$ is a stop time w.r.t. to the filtration $\calH^t$, the Optional Stopping Time ensures $\mathbb{E}[M^{\tau_{A}}]$ $=$ $\mathbb{E}[M^1]$ $=$ $0$, therefore:
\begin{align*}
    \mathbb{E} \left\lbrack  \sum_{t=1}^{\tau_{A}} \mathbb{E}_{w}\left[ f(z(\lambda^t;\theta^t,w);\theta^*,w) | \calH^{t-1} \right] \right\rbrack = \mathbb{E} \left\lbrack  \sum_{t=1}^{\tau_{A}} f(z(\lambda^t;\theta^t,w^t);\theta^*,w^t) \right\rbrack 
\end{align*}
concluding the proof.
\end{proof}

\begin{proposition}\label{Pro:KeySlater}
Run Algorithm \ref{Alg:DualMirrorLearning} with a constant ``step-size'' rule $\eta_t \gets \eta$ for all $t \geq 1$ where $\eta > 0$. Using $\delta_{\theta}$ as in Definition \ref{Def:SlackTheta}, for each $t \in [T-1]$ it holds (here we use $0$ to refer to the zero-vector  $(0,\dots, 0)$ of dimension $K$):
\begin{align*}
    \mathbb{E} \left\lbrack  V_h(0,\lambda^{t+1}) \big| \lambda^t,\theta^t \nonumber \right\rbrack \le \eta \left( \frac{2 \eta}{\sigma_1}(\bar{C}^2 +\bar{b}^2) + 2 \bar{f} - \delta_{\theta^t} \lVert \lambda^t \rVert_1 \right) + V_h(0,\lambda^t).
\end{align*}
\end{proposition}

\begin{proof}
Let  $\tilde{g}^t$ be the $\lambda^t$ stochastic subgradient obtained in Step 6. of Algorithm \ref{Alg:DualMirrorLearning}.  Here we abuse notation and use, \textit{e.g.}, $\mathbb{E}[\tilde{g}^t | \lambda^t, \theta^t]$ to represent that $\tilde{g}^t$ is a random variable on $w$ given a fixed pair $(\lambda^t,\theta^t) \in (\Lambda \times \Theta)$. The following bound holds
$$ \mathbb{E}_{\calP}[\lVert \tilde{g}^t \rVert_\infty^2] \le \mathbb{E}[(\lVert c(z^t;\theta^t,w^t) \rVert_\infty + \lVert b \rVert_\infty)^2] \le  2 \mathbb{E}[\lVert c(z^t;\theta^t,w^t) \rVert_\infty^2 + \lVert b \rVert_\infty^2] \le 2(\bar{C}^2+\bar{b}^2).$$
For any $\lambda \in \Lambda$ we have
    \begin{align}
        & \mathbb{E}[  \tilde{g}^t |\lambda^t,\theta^t]^T (\lambda^t-\lambda ) \nonumber \\
        = & \mathbb{E}[ ( \tilde{g}^t )^T (\lambda^t-\lambda ) |\lambda^t,\theta^t] \nonumber \\
        \le & \mathbb{E} \left\lbrack ( \tilde{g}^t )^T (\lambda^t-\lambda^{t+1}  ) + \frac{1}{\eta} V_h(\lambda,\lambda^t) -  \frac{1}{\eta} V_h(\lambda,\lambda^{t+1}) - \frac{1}{\eta} V_h(\lambda^{t+1},\lambda^{t}) \big| \lambda^t,\theta^t \right\rbrack  \nonumber \\
        \le & \mathbb{E} \left\lbrack ( \tilde{g}^t )^T ( \lambda^t - \lambda^{t+1} ) + \frac{1}{\eta} V_h(\lambda,\lambda^t) - \frac{1}{\eta} V_h(\lambda,\lambda^{t+1}) - \frac{\sigma_1}{2\eta} \lVert \lambda^{t+1} -\lambda^t\rVert_1^2 \big| \lambda^t,\theta^t  \right\rbrack \nonumber \\
        \le & \mathbb{E} \left\lbrack \frac{\eta}{\sigma_1}\lVert \tilde{g}^t \rVert_\infty^2 + \frac{1}{\eta}V_h(\lambda,\lambda^t) - \frac{1}{\eta}V_h(\lambda,\lambda^{t+1}) \big| \lambda^t,\theta^t  \right\rbrack \nonumber \\
        \le & \frac{2 \eta}{\sigma_1}(\bar{C}^2 +\bar{b}^2) + \frac{1}{\eta} V_h(\lambda,\lambda^t) -  \mathbb{E} \left\lbrack \frac{1}{\eta}  V_h(\lambda,\lambda^{t+1}) \big| \lambda^t,\theta^t \nonumber \right\rbrack,
    \end{align}
where we have used linearity of the expectation,  the three point property, that $V_h(\cdot,\cdot)$ is $\sigma_1$ strongly convex on with respect to the $\lVert \cdot \rVert_1$ norm, Cauchy-Schwartz, and the bound for $\mathbb{E}[\lVert \tilde{g}^t \rVert_\infty^2]$ obtained before (same steps as in Step 1. and 2. of Proof \ref{Pro:CorrectRegretUpToTau_A}).  Choosing $\lambda = (0, \dots, 0)$ we get 
\begin{align*}
    \mathbb{E} \left\lbrack  V_h(0,\lambda^{t+1}) \big| \lambda^t,\theta^t \nonumber \right\rbrack \le \eta \left( \frac{2 \eta}{\sigma_1}(\bar{C}^2 +\bar{b}^2) - \mathbb{E}[  \tilde{g}^t |\lambda^t,\theta^t]^T \lambda^t \right) + V_h(0,\lambda^t). 
\end{align*}
To finish the proof we now show that $\mathbb{E}[  \tilde{g}^t |\lambda^t,\theta^t]^T \lambda^t \ge \lVert \lambda^t \rVert_1 \delta_{\theta^t} - 2 \bar{f}$. Notice first that for any $(\lambda^t,\theta^t) \in (\Lambda \times \Theta)$ we have $\mathbb{E}[\tilde{g}^t(w)]^T \lambda^t  $ $=$ $ - \mathbb{E}[c(z(\lambda^t;\theta^t,w);\theta^t,w)]^T \lambda^t + p(\lambda^t) $ using that by definition $p(\lambda) = \sum_{k \in [K]} b_k ([\lambda_k]_+ - \alpha_k [-\lambda_k]_+ )$. Let $\{z(w)\}_{w \in \calW}$ be a series that satisfies 
$\delta_{\theta^t} = \mathbb{E}_{\calP} [\min \{ \lVert Tb_k - c_k(z(w);\theta^t,w) \rVert_\infty, \lVert c_k(z(w);\theta^t,w) - T\alpha_k b_k \rVert_\infty \}]$.  Then,
\begin{align*}
    & \mathbb{E}[  \tilde{g}^t |\lambda^t,\theta^t]^T \lambda^t   \\ 
    & = D(\lambda^t;\theta^t) - \mathbb{E}_{\calP}[f(z(\lambda^t;\theta^t,w);\theta^t,w)] \\
    & \ge \mathbb{E}_{\calP}[\max_{z\in \calZ} \, f(z;\theta^t,w) + \sum_{k \in [K]} \left([\lambda_k^t]_+ (b_k - \mathbb{E}_{\calP}[c_k(z;\theta^t,w)]) + [-\lambda_k^t]_+ ( \mathbb{E}_{\calP}[c_k(z;\theta^t,w)] - \alpha_k b_k) \right)] -\bar{f} \\
    & \ge \mathbb{E}_{\calP}[f(z(w);\theta^t,w) + \sum_{k \in [K]} \left([\lambda_k^t]_+ (b_k - \mathbb{E}_{\calP}[c_k(z(w);\theta^t,w)]) + [-\lambda_k^t]_+ ( \mathbb{E}_{\calP}[c_k(z(w);\theta^t,w)] - \alpha_k b_k) \right)] -\bar{f} \\
    & \ge \mathbb{E}_{\calP}[\sum_{k \in [K]} [\lambda_k^t]_+ (b_k - \mathbb{E}_{\calP}[c_k(z(w);\theta^t,w)]) + [-\lambda_k^t]_+ ( \mathbb{E}_{\calP}[c_k(z(w);\theta^t,w)] - \alpha_k b_k) ] - 2 \bar{f} \\
    & \ge \lVert \lambda^t \rVert_1 \delta_{\theta^t} - 2 \bar{f},
\end{align*}
where we have used the definition of $D(\lambda^t;\theta^t)$, $\bar{f}$, $\delta_{\theta^t}$, and the fact that $\lVert \lambda^t \rVert_1 = \sum_{k \in [K]}$ $([\lambda_k^t]_+ + [-\lambda_k^t]_+)$. 
\end{proof}

\subsection{Proof of Theorem \ref{Thm:Master}}
\begin{proof}
For any distribution $\calP$ over $\calW$ and for any $t' \in [T]$ we have
    \begin{align*}
        OPT(\calP) & \le \frac{t'}{T} OPT(\calP) + \frac{T-t'}{T} OPT(\calP) \\
        & \le t' D(\bar{\lambda}^{t'};\theta^*) + (T - t')\bar{f},
    \end{align*}
    where we have used Proposition \ref{Pro:WeakDuality} and that a loose upper bound for $OPT(\calP)$ is $T \bar{f}$. Therefore,
    \begin{align*}
        & Regret(A|\calP) \\
        = & OPT(\calP) - R(A|\calP)  \\
        \le & \mathbb{E} \left\lbrack \tau_A D(\bar{\lambda}^{\tau_A};\theta^*) + (T - \tau_A) \bar{f} -\sum_{t=1}^{\tau_A} f(z^t;\theta^*,w^t) \right\rbrack \\
        = & \mathbb{E} \left\lbrack \tau_A D(\bar{\lambda}^{\tau_A};\theta^*)  -\sum_{t=1}^{\tau_A} f(z^t;\theta^*,w^t)  \right\rbrack +    \mathbb{E}[T - \tau_A]\bar{f}  \\
        \le & \frac{ 2 (\bar{C}^2 +\bar{b}^2)}{\sigma_1} \eta \mathbb{E}[\tau_A] + \frac{1}{\eta} V_h(0,\lambda^1) + \frac{\bar{f}}{\underline{b}} \left( \bar{C} + \frac{C_h + \lVert \nabla h(\lambda^1) \rVert_\infty}{\eta }  \right)  \\
        + & \mathbb{E} \left\lbrack \sum_{t=1}^{\tau_A} (c(z^t;\theta^*,w^t) - c(z^t;\theta^t,w^t))^T\lambda^t  \right\rbrack  + \frac{\bar{f}}{\underline{b}} \left\lVert \mathbb{E} \left[  \sum_{t=1}^{\tau_A} c(z^t;\theta^*,w^t) - c(z^t;\theta^t,w^t) \right] \right\rVert_\infty,
    \end{align*}
    where in the first inequality we have used the definition of $R(A|\calP)$ and the fact that Algorithm \ref{Alg:DualMirrorLearning} runs for $\tau_A$ periods. The second inequality is obtained directly from Propositions \ref{Pro:StopTimeBound} and \ref{Pro:CorrectRegretUpToTau_A}. 
\end{proof}

\subsection{Proof of Proposition \ref{Prop:DualQuantityBounded}}
\begin{proof}
A direct application of Proposition \ref{Pro:KeySlater} shows that whenever $\lVert \lambda^t \rVert_1 \ge C^\rhd / \delta$ we have \\ \noindent $\mathbb{E}[V_h(0, \lambda^{t+1}) | (\lambda^t, \theta^t) ] \le V_h(0, \lambda^t)$. Then, for any $(\lambda^t, \theta^t) \in \Lambda \times \Theta$ we have
\begin{align*}
    & \mathbb{E}[V_h(0, \lambda^{t+1}) | (\lambda^t, \theta^t)] \le \max \big\lbrace \max_{\lVert \lambda \rVert_1 \le \delta^{-1} C^\rhd } \, V_h(0, \lambda) + \eta C^\rhd ,\, V_h(0,\lambda^1)  \big\rbrace \\
    \Rightarrow & \mathbb{E}[V_h(0, \lambda^{t+1})]  \le \max \big\lbrace \max_{\lVert \lambda \rVert_1 \le \delta^{-1} C^\rhd } \, V_h(0, \lambda) + \eta C^\rhd ,\, V_h(0,\lambda^1)  \big\rbrace
\end{align*}
Take now $h(\cdot) = \tfrac{1}{2} \lVert \cdot \rVert_2^2$, then for any $\lambda \in \Lambda$ we have $\nabla h(\lambda) = \lambda$ and $V_h(0, \lambda) = \tfrac{1}{2} \lVert \lambda \rVert_2^2$, therefore $ \max_{\lVert \lambda \rVert_1 \le \delta^{-1} C^\rhd } 0.5 \lVert \lambda \rVert_2^2 = 0.5  (C^\rhd/\delta)^2$. Using Jensen inequality and previous results we get 
\begin{align*}
    \mathbb{E}[\lVert \lambda^{t+1}\rVert_2]  \le \max \big\lbrace \sqrt{   (C^\rhd/\delta)^2 + 2 \eta C^{\rhd}}, \lVert \lambda^1 \rVert_2 \big\rbrace
\end{align*}
Finally, since $\lVert \lambda \rVert_\infty \le \lVert \lambda \rVert_2$ for any  $\lambda \in \Lambda$ is immediate that $ \mathbb{E} [\lVert \lambda^t \rVert_\infty]  \le \max \big\lbrace \sqrt{  (C^\rhd/\delta)^2 + 2 \eta C^{\rhd}}, \lVert \lambda^1 \rVert_\infty \big\rbrace$ for all $t \in [T]$ concluding the proof.

\end{proof}

\subsection{Proof of Proposition \ref{Pro:LowerBoundsSatisfied}}

\begin{proof}
Since $\alpha_k \ne -\infty$ by statement,  Proposition \ref{Pro:KKT_for_h} shows  $\dot{h}_k(\lambda^{t+1}) = \dot{h}_k(\lambda^{t}) - \eta \tilde{g}_k^t$ for any $t \in [T]$,  which implies that $\dot{h}_k(\lambda^{\tau_A+1}) - \dot{h}_k(\lambda^1)  = -\eta \sum_{t=1}^{\tau_A} \tilde{g}_k^t$ regardless of the $\tau_A$ value. Then, using the definition of $\tilde{g}^t$ we get
\begin{align*}
     &\sum_{t=1}^{\tau_A}  \left(b_k(\mathds{1}(\lambda_k \ge 0) + \alpha_k \mathds{1}(\lambda_k < 0)) - c_k(z^t;\theta^t,w^t) \right)
    =  \frac{\dot{h}_k(\lambda^1) - \dot{h}_k(\lambda^{\tau_A+1})}{\eta} \\
    \Rightarrow & \sum_{t=1}^{\tau_A}  \left(b_k(\mathds{1}(\lambda_k \ge 0) + \alpha_k \mathds{1}(\lambda_k < 0)) - c_k(z^t;\theta^*,w^t) \right)  =  \frac{\dot{h}_k(\lambda^1) - \dot{h}_k(\lambda^{\tau_A+1})}{\eta} +  \sum_{t=1}^{\tau_A} c_k(z^t;\theta^t,w^t) - c_k(z^t;\theta^*, w^t).
\end{align*}
Now, given that $(\mathds{1}(\lambda' \ge 0) +\alpha_k \mathds{1}(\lambda' <0)) \ge \alpha_k$ for any $\lambda' \in \bbR$ and that $\tau_A \le T$ by definition, we have
\begin{align*}
    \sum_{t=1}^{\tau_A}  \left(b_k(\mathds{1}(\lambda_k \ge 0) + \alpha_k \mathds{1}(\lambda_k < 0)) \right) + (T - \tau_A ) \alpha_k b_k  \ge T \alpha_k b_k.
\end{align*}
Combining the previous results and taking expectation we get
\begin{align*}
T \alpha_k b_k - \mathbb{E}[\sum_{t=1}^{\tau_A}c_k(z^t;\theta^*,w^t)]  \le \frac{\dot{h}_k(\lambda^1) - \mathbb{E}[\dot{h}_k(\lambda^{\tau_A+1})]}{ \eta} + \mathbb{E}[T-\tau_A] \alpha_k b_k  + \mathbb{E}\left\lbrack \sum_{t=1}^{\tau_A} c_k(z^t;\theta^t,w^t) - c_k(z^t;\theta^*, w^t) \right\rbrack.
\end{align*}
Finally, we conclude the proof by using Proposition \ref{Pro:StopTimeBound} and the definition of $C_h$.
\end{proof}

\section{Extra Experimental Details and Results}

\subsection{Bidding Experiment}

This experiment is based on data from Criteo \citet{diemert2017attribution}. Criteo is a Demand-Side Platform (DSP), which are entities which bid on behalf of hundreds or thousands of advertisers which set campaigns with them. The dataset from \citet{diemert2017attribution} contains millions of bidding logs  during one month of Criteo's operation. These bidding logs are all logs in which Criteo successfully acquired ad-space for its advertising clients through real-time second-price auctions (each log represents a different auction and ad-space). Each of these auctions occur when a user arrives to a website, app, etc., and each user is shown one ad few millisecond after its ``arrival''. Each bidding log contains. 1. Nine anonymized categorical columns containing characteristics of the ad-space and (possibly) about the user who has just ``arrived''. 2. The price Criteo paid for the ad-space, which corresponds to the second highest bid submitted to each auction. 3. The day of the auction and the advertiser whose ad was shown in the ad-space (the day is not included directly in the dataset, but appears in a Jupyter Notebook inside the compressed file that contains the dataset). 4. If a conversion occur after the ad was shown, \textit{i.e.}, if the corresponding user performed an action of interest for the advertiser after watching the advertiser's ad. The dataset can be downloaded from \url{https://ailab.criteo.com/criteo-attribution-modeling-bidding-dataset}.

The experiment was performed as follows. We used the first 21 days of data as training, the next two days as validation, and the remaining seven days as test. The training data was used only to train a neural network to predict the probability of a conversion occurring. The  model architecture was taken from \citet{pan2018field} and uses as features the nine anonymized categorical columns, the day of the week, and an advertiser id to make a prediction if a conversion would occur or not. Parameters to be tuned for the neural network were the step-size for the Adam solver, embedding sizes, and other two specific network attributes (in total we tried 120 configurations). Once found the trained model with highest validation AUC (Area Under the Curve), we took this model predictions as if they were the real probabilities of a conversion occurring for unseen data. By having the advertiser id as an input on the model, we can get conversion probability estimates for all advertisers even when Criteo bid on behalf of only one advertiser per bidding log. The advertisers pay the DSP, in our context the bidder, each time the DSP bids on behalf of them. The payment corresponds to the probability of conversion times a known fixed value. The general simulator scheme for this experiment is shown in Algorithm \ref{Alg:2ndSimulator}.

\floatname{algorithm}{Algorithm} 
\begin{algorithm} 
\caption{Simulator Scheme} \label{Alg:2ndSimulator}
\begin{algorithmic}
\STATE {\textbf{Input:}} Trained conversion prediction model $\sigma$, the set of all test bidding logs $X_{test}$, $T$ the number of test bidding logs, $q \in \bbR_+^K$ the vector of payment per conversion values for the advertisers, $\{\mathrm{mp}^t\}_{t=1}^T$ the price Criteo paid for each ad spot in the test set in order. 
\FOR{$t = 1,\dots,T$}
\STATE 1. Read the $t$ bidding test log and $\mathrm{mp}^t$. 
\STATE 2. Use model $\sigma$ to obtain estimated conversion probabilities $\mathrm{conv\_prob}$. Take $r_k^t = \mathrm{conv\_prob}_k \cdot q_k$ for all $k \in K$.
\STATE 3. Using vector $r^t$ and previous history, obtain $(z^t,k^t)$ a pair of submitted bid and advertiser to bid on behalf of.
\STATE 4. If $z^t \ge \mathrm{mp}^t$ then the auction is won, advertiser $k^t$ pays $r_{k^t}^t$ to the bidder (the DSP), the bidder pays $\mathrm{mp}^t$ for the ad spot and obtains $r_{k^t}^t - \mathrm{mp}^t$ as profit.
\ENDFOR
\end{algorithmic}
\end{algorithm}

Algorithm \ref{Alg:DualMirrorLearning} can be naturally incorporated in the simulator scheme by using the online optimization component of it to obtain $(z^t,k^t)$ of Step 3. of the simulator. We only need the online optimization component of Algorithm \ref{Alg:DualMirrorLearning}, as we do not need to learn the distribution of the highest competing ($\mathrm{mp}$) to solve Step 3. of Algorithm \ref{Alg:DualMirrorLearning} (shown in Algorithm \ref{Alg:SecPrice}). We compare the performance of Algorithm \ref{Alg:DualMirrorLearning} to using the Greedy Heuristic \ref{Alg:GreedyHeuristic}. When $\gamma =1$, Algorithm \ref{Alg:GreedyHeuristic} bids `truthfully' on behalf of the advertiser with the highest valuation. This would be the optimal strategy if the advertisers had `infinite' budgets and no lower bound requirements. Then, we can think of $\gamma$ as a way to increase/decrease the bids in order to take the budgets into account. (For this example, we can think of Algorithm \ref{Alg:DualMirrorLearning} as an online algorithm for obtaining $\gamma$ variables per advertiser.)

\floatname{algorithm}{Algorithm} 
\begin{algorithm} 
\caption{Greedy Heuristic($\gamma$)} \label{Alg:GreedyHeuristic}
\begin{algorithmic}
\STATE {\textbf{Input:}} Vector $r \in \bbR_+^K$ and $\gamma > 0$.  
\STATE Let $\calK'$ be the set of advertisers with non depleted budgets. If $\calK' = \emptyset$ do not bid, otherwise bid on behalf of $k^* \in \arg\max_{k \in \calK'}$ $r_k$ the amount $\gamma r_{k^*}$. 
\end{algorithmic}
\end{algorithm}

Our test set contains 21073 iterations and 130 advertisers. (The original dataset had 700 advertisers but we removed all advertisers who appeared in less than 10,000 logs in either the training or validation plus test data.) Each iteration of the simulator scheme uses a batch of 128 test logs.  The total budget of an advertiser is the total amount Criteo spent bidding on behalf of that advertiser in the test logs multiplied by 100. We run Algorithm \ref{Alg:DualMirrorLearning} using traditional subgradient descent trying the fixed step sizes $\{1*10^{-i}\}_{i=0}^3 \cup \{ 0.5 * 10^{-i} \}_{i=0}^3$ and $\{ 0.25 + 0.05 *i\}_{i=0}^{25}$ as $\gamma$ parameters for the Greedy Heuristic \ref{Alg:GreedyHeuristic}. We run 100 simulations for each parameter and method pair. Each simulation is defined by the price advertisers would pay per conversion, which is the $q$ vector in Algorithm \ref{Alg:2ndSimulator}. We sample $q_k$ i.i.d. from $\mathrm{Uniform}(0.5, 1.5)$ for all $k \in [K]$. We relaxed the ending condition of Algorithm \ref{Alg:DualMirrorLearning} by allowing advertisers to overspend at most on one iteration. After that iteration, we consider an advertiser's budget as depleted and do not bid on behalf of it until the simulation's end. The final parameters chosen for Algorithms \ref{Alg:DualMirrorLearning} and \ref{Alg:GreedyHeuristic} were those that achieved the highest average profit.

An advertiser's budget depletion time correlates with its relative total maximum budget, fact that is shown in  Figure \ref{Fig:TimeVsRelativeBudget}. The x-axis is in logarithmic scale and shows the proportion of an advertiser budget w.r.t. the highest budget between all advertisers. Observe that as the relative budget increases, the average depletion time gets closer to the simulation end ($T =21073$).  

\begin{figure}
\centering
    \includegraphics[width=0.6\textwidth]{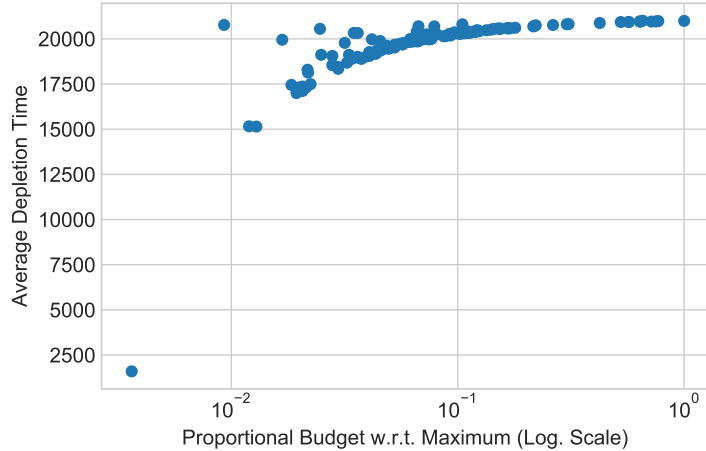}
    \caption{The x-axis in the figure shows the proportion of an advertiser budget w.r.t. the highest budget between all advertisers (shown on a logarithmic scale). \label{Fig:TimeVsRelativeBudget}}
\end{figure}

Finally, we run this experiment using a SLURM managed Linux cluster. We tried 120 parameters combinations for the conversion prediction architecture, running each parameter configuration for 25 epochs. Each parameter configuration took approximately 40 min to run using a Nvidia K80 GPU plus two Intel Xeon 4-core 3.0 Ghz (we used eight GPUs in parallel having a total run time of approximately 12 hours). For the experiment itself,  we tried nine different step-sizes to run the subgradient descent step using Algorithm \ref{Alg:DualMirrorLearning} and 26 $\gamma$ values for \ref{Alg:GreedyHeuristic},  each configuration running 100 simulations. We used several cluster nodes each having 64GB of RAM and two Xeon 12-core Haswell with 2.3 Ghz per core. If we had used just one node it would have taken approximately 160 hours to run all required configurations.

\subsection{Linear Contextual Bandits Experiment}\label{Subsec:ExtraExpResults}

We now describe in detail the methods used to implement Step 1. of Algorithm \ref{Alg:DualMirrorLearning}. First, let $y^t$ be the variable that takes the value of one if an action is taken at period $t$ and zero otherwise. Also, remember that $i(t) \in [d]$ is the action taken at period $t$ (if any), and $r^t$  the revenue observed at period $t$. We implemented Step 1. of Algorithm \ref{Alg:DualMirrorLearning} in the following ways.
\begin{enumerate}
\item Gaussian Thompson Sampling as in \citet{agrawal2013thompson}. Define $B(1) = I_d$ with $I_d$ the identity matrix of size $d$, and $\hat{\theta}^1 = (1/\sqrt{d},\dots,1/\sqrt{d})$. The Thompson Sampling procedure is composed of two steps which are updating a prior and sampling from a Gaussian posterior. We update the prior as follows. If $y^t =1$, make $B(t+1) = I_d + \sum_{s \in [t]: y^s = 1} W_{i(s)}^s (W_{i(s)}^s)^T$ and $\hat{\theta}^{t+1} = B(t+1)^{-1} (\sum_{s \in [t]: y^s = 1}  W_{i(s)}^s r^t)$, otherwise $B(t+1) = B(t)$ and $\hat{\theta}^{t+1} = \hat{\theta}^{t}$.  After the prior update, we sample $\theta^t$ from $\mathcal{N}(\hat{\theta}^{t}, \nu^2 B(t)^{-1})$ where $\mathcal{N}(\cdot,\cdot)$ represents a normal distribution defined by its mean and covariance matrix, and $\nu>0$ a constant chosen as follows. When no randomness was added to the observed revenue term, we used $\nu = 0.1$ (remember that we could add randomness to both the matrices $W^t$ and the observed revenue separately). When randomness was added to the observed revenue, we used $\nu = \tfrac{\mathrm{rev\_err}}{10} * \sqrt{\log{T} * n }$ with $\mathrm{rev\_err} = 0.1$ or $0.5$ depending if a $\mathrm{Uniform}(-0.1,0.1)$ or $\mathrm{Uniform}(-0.5,0.5)$ is added to the observed revenue term respectively. (The latter form of choosing $\nu$ was inspired on \citet{agrawal2013thompson} which uses $\nu = R \sqrt{9 n \log{T}}$ to prove a regret bound for Thompson Sampling for linear contextual bandits without constraints.)
\item Least squares. Same as Thompson Sampling as described above, but Step 1. of Algorithm \ref{Alg:DualMirrorLearning} uses $\theta^{t} = \hat{\theta}^t$. (This update is a core element of many learning approaches for linear contextual bandits \citet{agrawal2013thompson, agrawal2016linear} and can be understood as a Least Squares step). 
\item Ridge regression. We use the Least Squares procedure as defined above for the first $\sqrt{T}/2$ actions, and then solve a ridge regression problem. We solve a ridge regression problem at Step 1. of iteration $t$ using the set $\{ W_{i(s)}^s, r^s \}_{s \in [t-1]: y^s = 1}$ with an $\ell_2$ penalization parameter of $\alpha = 0.001$.
\item Ridge regression plus error. Same method as above but adds noise to the $\theta^t$ obtained from the ridge regression problem. We add an i.i.d. $\mathrm{Uniform}(-0.3,0.3)/\sqrt{\sum_{s=1}^t y^s}$ term to each coordinate of $\theta^t$.
\item Known $\theta^*$. Algorithm \ref{Alg:DualMirrorLearning} using $\theta^t = \theta^*$ for all $t \in [T]$.
\end{enumerate}


Figures \ref{Fig:MvAvg5_10} and \ref{Fig:MvAvg50_50} show how the different methods perform for $(d \times n)$ being $(5,10)$ and $(50,50)$ when $T = 10,000$, respectively. Each element of the x-axis represents a moving window composed of $250$ iterations. The x-axis is composed of $9751$ ticks . The y-axis shows the average relative revenue obtained in a window with respect to the proportional best revenue that could have been obtained ($\mathrm{OPT}(\calP) \cdot \tfrac{250}{10000}$). Importantly, the number of actions a method takes can vary between windows, which explains the following two facts. First, an initial revenue spike as many actions are taken when a simulation starts. The latter occurs as we took $\lambda^1=0$ which makes the cost component in Step 3. of Algorithm \ref{Alg:DualMirrorLearning} zero. Second, a method may obtain a higher average revenue on a window than $ \mathrm{OPT}(\calP) \cdot \tfrac{250}{10000}$ if more than 'average' actions are taken on that window.

Tables \ref{Tab:AllResT1k}, \ref{Tab:AllResT5k}, \ref{Tab:AllRes10k} show the average total relative revenue obtained for the different combinations of $d \times n$ and uncertainty used with respect to $\mathrm{OPT}(\calP)$. In general, as long as the budget is spent properly, the revenue obtained by the `Known $\theta^*$' method when $W^t = W$ for all $t \in [T]$ should match $\mathrm{OPT}(\calP)$. The latter as the best action to take is always the same. In the case when we still have $W^t = W$ for all $t \in [T]$, but the observed revenue has randomness, the `Known $\theta^*$' method may obtain a higher total revenue than $\mathrm{OPT}(\calP)$.


Finally, we run this experiment using a SLURM managed Linux cluster. We used four nodes each having 64 GB of RAM and 20 cores of 2.5Ghz. We parallelized the code to run each combination of experiment setting and simulation number as a different run (the run-time was mostly spent on sampling from a Gaussian distribution for Thompson Sampling and solving Ridge Regression problems with thousands of points). The total running time was 12 hours. Processing the results was done in a local computer (Mac Book Pro 2015 version), spending around 30 minutes to aggregate the results obtained from the cluster.

\begin{figure}
    \includegraphics[width=1.0\textwidth]{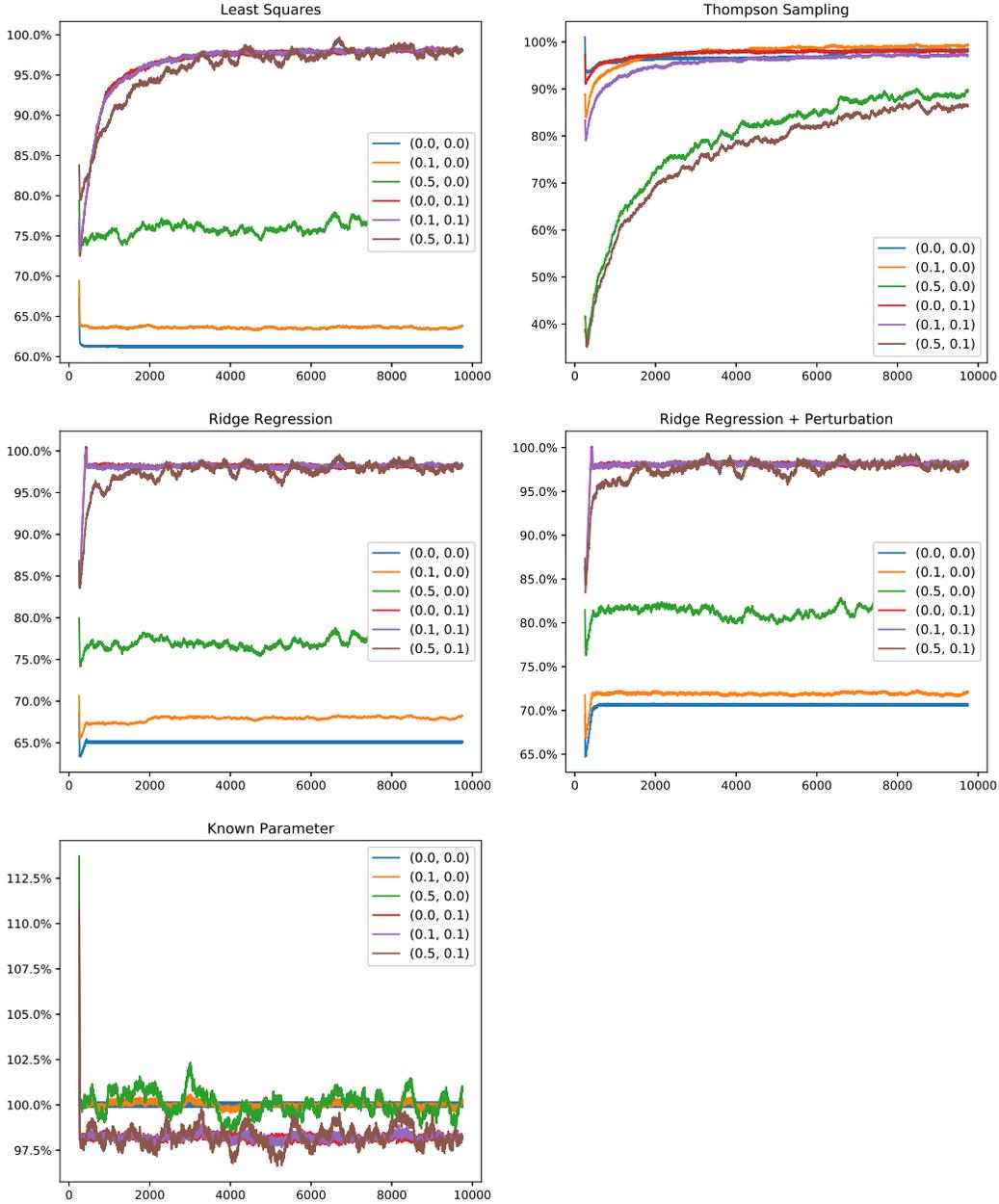}
    \vspace*{-15mm}
    \caption{Moving average revenue for windows of $250$ iterations against the proportional best average revenue possible using  $d =5$, $n = 10$. \label{Fig:MvAvg5_10}}
\end{figure}
\newpage

\begin{figure}
    \includegraphics[width=1.0\textwidth]{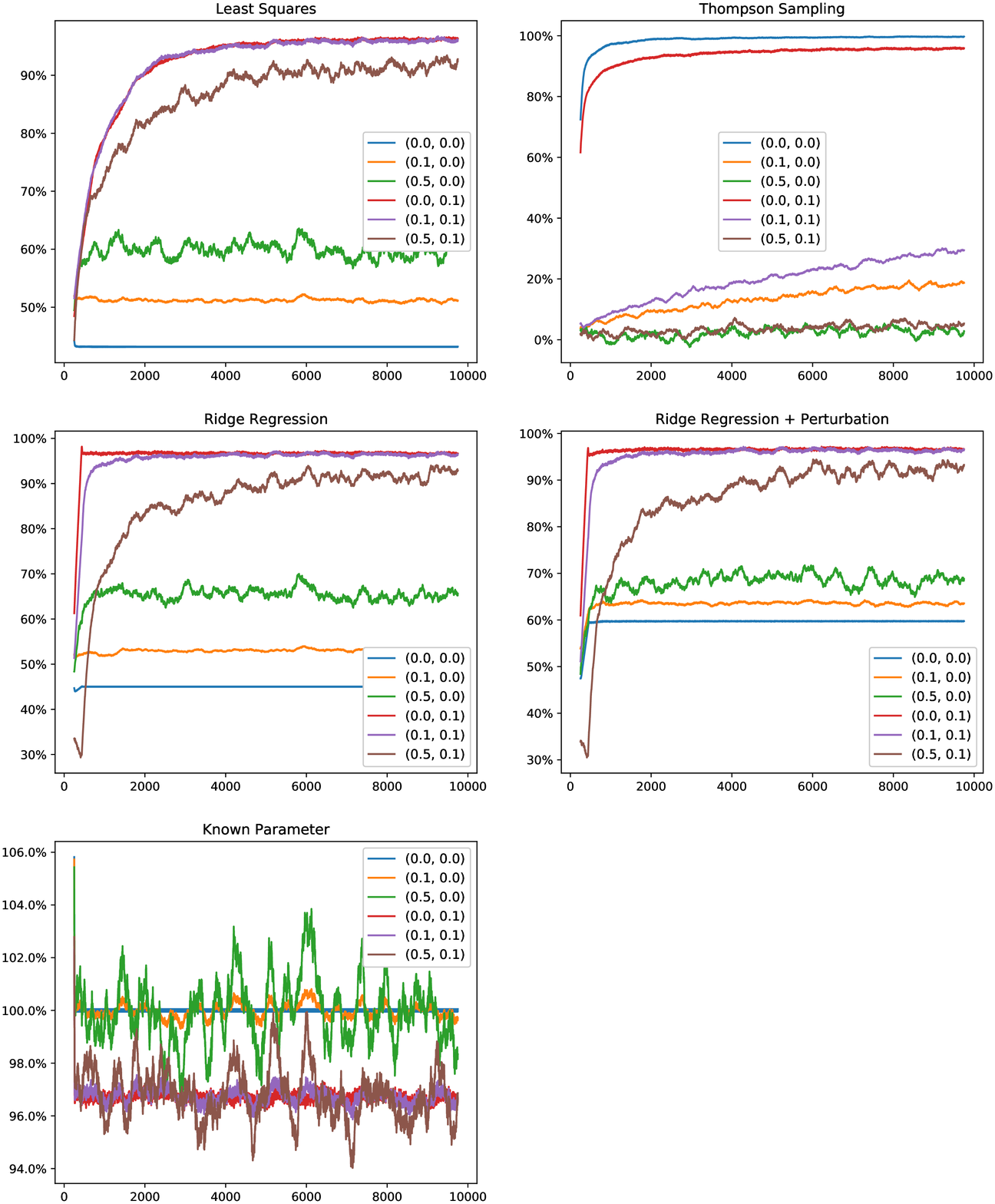}
    \vspace*{-15mm}
    \caption{Moving average revenue for windows of $250$ iterations against the proportional best average revenue possible using  $d =50$, $n = 50$. \label{Fig:MvAvg50_50}}
\end{figure}

\newpage

 \begin{table}[h]
    \begin{tabular}{| p{4.6 cm} | p{1.35 cm} | p{1.35 cm} | p{1.35 cm} | p{1.35 cm} | p{1.35 cm} | p{1.35 cm} | p{1.35 cm} |}
    \hline
    $T = 1,000$ & $d \times n$ &  $(0.0, 0.0)$ & $(0.1, 0.0)$ & $(0.5, 0.0)$ & $(0.0, 0.1)$ & $(0.1, 0.1)$ & $(0.5, 0.1)$ \\ \hline \hline
Least Squares & 5 $\times$ 5 & 76.2\% & 77.6\% & 84.2\% & 78.9\% & 79.4\% & 79.3\% \\ \hline
Thompson Sampling & 5 $\times$ 5 & 95.2\% & 74.2\% & 21.9\% & 85.4\% & 65.4\% & 19.0\% \\ \hline
Ridge Regression & 5 $\times$ 5 & 77.6\% & 79.0\% & 85.4\% & 90.4\% & 89.8\% & 83.5\% \\ \hline
Ridge Reg. + Perturbation & 5 $\times$ 5 & 80.8\% & 80.9\% & 86.0\% & 90.3\% & 89.6\% & 83.5\% \\ \hline
Known Parameter & 5 $\times$ 5 & 99.9\% & 100.1\% & 100.8\% & 92.4\% & 92.4\% & 92.2\% \\ \hline
 \hline
Least Squares & 5 $\times$ 10 & 60.9\% & 63.2\% & 73.4\% & 80.1\% & 80.5\% & 82.6\% \\ \hline
Thompson Sampling & 5 $\times$ 10 & 94.2\% & 90.3\% & 51.2\% & 89.4\% & 85.9\% & 48.3\% \\ \hline
Ridge Regression & 5 $\times$ 10 & 64.5\% & 67.3\% & 76.5\% & 93.0\% & 92.8\% & 90.0\% \\ \hline
Ridge Reg. + Perturbation & 5 $\times$ 10 & 73.9\% & 73.9\% & 81.1\% & 92.8\% & 92.6\% & 90.2\% \\ \hline
Known Parameter & 5 $\times$ 10 & 100.0\% & 100.0\% & 99.9\% & 95.5\% & 95.5\% & 95.4\% \\ \hline
 \hline
Least Squares & 10 $\times$ 5 & 70.9\% & 74.6\% & 78.1\% & 83.7\% & 84.0\% & 82.9\% \\ \hline
Thompson Sampling & 10 $\times$ 5 & 94.5\% & 91.0\% & 50.6\% & 89.0\% & 84.6\% & 47.6\% \\ \hline
Ridge Regression & 10 $\times$ 5 & 71.0\% & 75.2\% & 78.8\% & 92.5\% & 92.5\% & 89.0\% \\ \hline
Ridge Reg. + Perturbation & 10 $\times$ 5 & 82.0\% & 84.3\% & 84.7\% & 92.3\% & 92.3\% & 89.7\% \\ \hline
Known Parameter & 10 $\times$ 5 & 99.9\% & 99.9\% & 99.7\% & 94.5\% & 94.4\% & 94.2\% \\ \hline
 \hline
Least Squares & 10 $\times$ 10 & 58.5\% & 62.5\% & 72.0\% & 75.7\% & 75.3\% & 76.7\% \\ \hline
Thompson Sampling & 10 $\times$ 10 & 92.2\% & 66.6\% & 14.7\% & 86.1\% & 62.4\% & 15.1\% \\ \hline
Ridge Regression & 10 $\times$ 10 & 59.0\% & 63.4\% & 72.4\% & 91.2\% & 90.4\% & 84.0\% \\ \hline
Ridge Reg. + Perturbation & 10 $\times$ 10 & 72.3\% & 73.6\% & 77.2\% & 90.9\% & 90.2\% & 84.1\% \\ \hline
Known Parameter & 10 $\times$ 10 & 100.0\% & 99.9\% & 99.7\% & 93.9\% & 93.9\% & 93.9\% \\ \hline
 \hline
Least Squares & 25 $\times$ 25 & 44.0\% & 49.7\% & 54.0\% & 64.5\% & 66.0\% & 58.9\% \\ \hline
Thompson Sampling & 25 $\times$ 25 & 89.1\% & 5.4\% & 0.3\% & 74.4\% & 6.1\% & 0.7\% \\ \hline
Ridge Regression & 25 $\times$ 25 & 44.1\% & 50.6\% & 56.0\% & 86.1\% & 78.4\% & 46.8\% \\ \hline
Ridge Reg. + Perturbation & 25 $\times$ 25 & 69.5\% & 66.7\% & 61.4\% & 85.4\% & 78.0\% & 46.3\% \\ \hline
Known Parameter & 25 $\times$ 25 & 100.0\% & 100.0\% & 99.7\% & 90.7\% & 90.8\% & 91.3\% \\ \hline
 \hline
Least Squares & 25 $\times$ 50 & 41.4\% & 48.1\% & 56.1\% & 64.7\% & 65.1\% & 68.1\% \\ \hline
Thompson Sampling & 25 $\times$ 50 & 89.0\% & 19.6\% & 3.3\% & 82.4\% & 20.5\% & 3.7\% \\ \hline
Ridge Regression & 25 $\times$ 50 & 43.3\% & 50.3\% & 62.7\% & 90.0\% & 85.8\% & 69.7\% \\ \hline
Ridge Reg. + Perturbation & 25 $\times$ 50 & 62.8\% & 64.0\% & 68.8\% & 89.5\% & 85.5\% & 69.1\% \\ \hline
Known Parameter & 25 $\times$ 50 & 100.0\% & 100.1\% & 100.3\% & 93.7\% & 93.8\% & 94.1\% \\ \hline
 \hline
Least Squares & 50 $\times$ 25 & 49.1\% & 53.7\% & 59.1\% & 67.7\% & 68.1\% & 68.3\% \\ \hline
Thompson Sampling & 50 $\times$ 25 & 92.2\% & 18.3\% & 2.6\% & 82.7\% & 19.5\% & 2.8\% \\ \hline
Ridge Regression & 50 $\times$ 25 & 51.9\% & 55.9\% & 64.6\% & 89.4\% & 85.3\% & 67.7\% \\ \hline
Ridge Reg. + Perturbation & 50 $\times$ 25 & 70.8\% & 69.7\% & 71.8\% & 89.1\% & 85.2\% & 67.6\% \\ \hline
Known Parameter & 50 $\times$ 25 & 100.0\% & 100.0\% & 100.0\% & 92.9\% & 92.9\% & 92.6\% \\ \hline
 \hline
Least Squares & 50 $\times$ 50 & 42.0\% & 52.2\% & 55.7\% & 62.3\% & 63.7\% & 58.7\% \\ \hline
Thompson Sampling & 50 $\times$ 50 & 87.5\% & 5.4\% & 1.5\% & 76.0\% & 6.7\% & 1.5\% \\ \hline
Ridge Regression & 50 $\times$ 50 & 43.6\% & 54.5\% & 62.1\% & 86.8\% & 76.8\% & 47.7\% \\ \hline
Ridge Reg. + Perturbation & 50 $\times$ 50 & 67.2\% & 68.8\% & 66.7\% & 86.0\% & 76.6\% & 47.2\% \\ \hline
Known Parameter & 50 $\times$ 50 & 100.0\% & 100.0\% & 100.0\% & 92.0\% & 91.9\% & 91.6\% \\ \hline
    \end{tabular}
    \caption{All percentages shown are the average revenue over 100 simulations divided by the best average revenue achievable ($\mathrm{OPT}(\calP)$).  \label{Tab:AllResT1k} }
\end{table}
 
\newpage
 
  \begin{table}[h]
    \begin{tabular}{| p{4.6 cm} | p{1.35 cm} | p{1.35 cm} | p{1.35 cm} | p{1.35 cm} | p{1.35 cm} | p{1.35 cm} | p{1.35 cm} |}
    \hline
    $T = 5,000$ & $d \times n$ &  $(0.0, 0.0)$ & $(0.1, 0.0)$ & $(0.5, 0.0)$ & $(0.0, 0.1)$ & $(0.1, 0.1)$ & $(0.5, 0.1)$ \\ \hline \hline
Least Squares & 5 $\times$ 5 & 76.7\% & 79.4\% & 87.1\% & 91.6\% & 91.5\% & 90.5\% \\ \hline
Thompson Sampling & 5 $\times$ 5 & 98.7\% & 88.6\% & 42.6\% & 93.2\% & 80.9\% & 36.7\% \\ \hline
Ridge Regression & 5 $\times$ 5 & 78.1\% & 79.4\% & 86.5\% & 95.1\% & 94.9\% & 92.4\% \\ \hline
Ridge Reg. + Perturbation & 5 $\times$ 5 & 80.0\% & 79.7\% & 87.2\% & 94.9\% & 94.8\% & 92.3\% \\ \hline
Known Parameter & 5 $\times$ 5 & 100.0\% & 100.0\% & 99.9\% & 95.9\% & 95.9\% & 96.0\% \\ \hline
 \hline
Least Squares & 5 $\times$ 10 & 61.2\% & 63.5\% & 75.3\% & 93.1\% & 93.3\% & 92.6\% \\ \hline
Thompson Sampling & 5 $\times$ 10 & 97.3\% & 96.0\% & 71.7\% & 95.8\% & 93.0\% & 68.6\% \\ \hline
Ridge Regression & 5 $\times$ 10 & 64.9\% & 67.9\% & 79.6\% & 96.5\% & 96.5\% & 95.5\% \\ \hline
Ridge Reg. + Perturbation & 5 $\times$ 10 & 71.0\% & 71.9\% & 80.4\% & 96.4\% & 96.4\% & 95.3\% \\ \hline
Known Parameter & 5 $\times$ 10 & 100.0\% & 100.0\% & 100.0\% & 97.5\% & 97.5\% & 97.4\% \\ \hline
 \hline
Least Squares & 10 $\times$ 5 & 71.3\% & 72.3\% & 80.9\% & 93.6\% & 93.4\% & 93.4\% \\ \hline
Thompson Sampling & 10 $\times$ 5 & 96.0\% & 96.4\% & 70.4\% & 95.2\% & 92.1\% & 67.1\% \\ \hline
Ridge Regression & 10 $\times$ 5 & 71.5\% & 73.7\% & 81.5\% & 96.3\% & 96.2\% & 95.5\% \\ \hline
Ridge Reg. + Perturbation & 10 $\times$ 5 & 77.0\% & 80.1\% & 83.0\% & 96.2\% & 96.1\% & 95.3\% \\ \hline
Known Parameter & 10 $\times$ 5 & 100.0\% & 100.0\% & 100.1\% & 97.0\% & 97.0\% & 97.0\% \\ \hline
 \hline
Least Squares & 10 $\times$ 10 & 58.9\% & 63.3\% & 70.0\% & 91.0\% & 90.9\% & 91.3\% \\ \hline
Thompson Sampling & 10 $\times$ 10 & 96.2\% & 83.9\% & 29.5\% & 94.2\% & 80.7\% & 30.8\% \\ \hline
Ridge Regression & 10 $\times$ 10 & 59.4\% & 63.7\% & 70.4\% & 95.6\% & 95.4\% & 93.3\% \\ \hline
Ridge Reg. + Perturbation & 10 $\times$ 10 & 69.2\% & 69.8\% & 74.1\% & 95.5\% & 95.4\% & 93.1\% \\ \hline
Known Parameter & 10 $\times$ 10 & 100.0\% & 100.0\% & 100.1\% & 96.7\% & 96.6\% & 96.5\% \\ \hline
 \hline
Least Squares & 25 $\times$ 25 & 44.6\% & 54.0\% & 58.6\% & 85.6\% & 85.6\% & 78.3\% \\ \hline
Thompson Sampling & 25 $\times$ 25 & 97.2\% & 12.6\% & 1.2\% & 88.6\% & 15.0\% & 1.9\% \\ \hline
Ridge Regression & 25 $\times$ 25 & 44.8\% & 54.7\% & 60.4\% & 93.4\% & 91.1\% & 76.4\% \\ \hline
Ridge Reg. + Perturbation & 25 $\times$ 25 & 64.9\% & 64.0\% & 66.2\% & 93.2\% & 90.9\% & 76.5\% \\ \hline
Known Parameter & 25 $\times$ 25 & 100.0\% & 100.1\% & 100.4\% & 95.0\% & 94.9\% & 94.7\% \\ \hline
 \hline
Least Squares & 25 $\times$ 50 & 41.5\% & 48.1\% & 57.5\% & 87.7\% & 87.4\% & 84.4\% \\ \hline
Thompson Sampling & 25 $\times$ 50 & 94.6\% & 36.2\% & 7.3\% & 93.0\% & 39.9\% & 8.6\% \\ \hline
Ridge Regression & 25 $\times$ 50 & 43.5\% & 49.9\% & 68.0\% & 95.0\% & 94.2\% & 87.8\% \\ \hline
Ridge Reg. + Perturbation & 25 $\times$ 50 & 55.7\% & 58.0\% & 74.1\% & 94.9\% & 94.1\% & 87.0\% \\ \hline
Known Parameter & 25 $\times$ 50 & 100.0\% & 99.9\% & 99.6\% & 96.5\% & 96.5\% & 96.5\% \\ \hline
 \hline
Least Squares & 50 $\times$ 25 & 49.3\% & 53.0\% & 57.8\% & 87.6\% & 87.9\% & 85.3\% \\ \hline
Thompson Sampling & 50 $\times$ 25 & 97.8\% & 34.3\% & 5.5\% & 92.3\% & 38.9\% & 7.1\% \\ \hline
Ridge Regression & 50 $\times$ 25 & 52.2\% & 55.3\% & 58.4\% & 94.6\% & 93.9\% & 86.8\% \\ \hline
Ridge Reg. + Perturbation & 50 $\times$ 25 & 66.0\% & 65.7\% & 67.8\% & 94.4\% & 93.7\% & 87.1\% \\ \hline
Known Parameter & 50 $\times$ 25 & 100.0\% & 100.0\% & 100.1\% & 96.0\% & 96.0\% & 96.0\% \\ \hline
 \hline
Least Squares & 50 $\times$ 50 & 41.9\% & 52.7\% & 60.4\% & 85.8\% & 86.2\% & 79.6\% \\ \hline
Thompson Sampling & 50 $\times$ 50 & 96.4\% & 10.0\% & 1.8\% & 89.7\% & 14.3\% & 2.7\% \\ \hline
Ridge Regression & 50 $\times$ 50 & 43.6\% & 53.2\% & 68.2\% & 94.0\% & 91.5\% & 77.9\% \\ \hline
Ridge Reg. + Perturbation & 50 $\times$ 50 & 59.9\% & 61.3\% & 71.8\% & 93.7\% & 91.4\% & 77.8\% \\ \hline
Known Parameter & 50 $\times$ 50 & 100.0\% & 100.0\% & 100.2\% & 95.5\% & 95.5\% & 95.5\% \\ \hline
    \end{tabular}
    \caption{All percentages shown are the average revenue over 100 simulations divided by the best average revenue achievable ($\mathrm{OPT}(\calP)$).   \label{Tab:AllResT5k} }
\end{table}

\newpage

 \begin{table}[h]
    \begin{tabular}{| p{4.6 cm} | p{1.35 cm} | p{1.35 cm} | p{1.35 cm} | p{1.35 cm} | p{1.35 cm} | p{1.35 cm} | p{1.35 cm} |}
    \hline
    $T = 10,000$ & $d \times n$ &  $(0.0, 0.0)$ & $(0.1, 0.0)$ & $(0.5, 0.0)$ & $(0.0, 0.1)$ & $(0.1, 0.1)$ & $(0.5, 0.1)$ \\ \hline \hline
Least Squares & 5 $\times$ 5 & 76.8\% & 79.7\% & 85.4\% & 94.7\% & 94.6\% & 93.7\% \\ \hline
Thompson Sampling & 5 $\times$ 5 & 98.8\% & 92.4\% & 52.8\% & 95.4\% & 85.8\% & 47.0\% \\ \hline
Ridge Regression & 5 $\times$ 5 & 78.2\% & 79.7\% & 87.0\% & 96.5\% & 96.4\% & 95.0\% \\ \hline
Ridge Reg. + Perturbation & 5 $\times$ 5 & 80.1\% & 80.0\% & 88.6\% & 96.4\% & 96.4\% & 95.0\% \\ \hline
Known Parameter & 5 $\times$ 5 & 100.0\% & 100.0\% & 100.2\% & 97.0\% & 97.0\% & 97.1\% \\ \hline
 \hline
Least Squares & 5 $\times$ 10 & 61.2\% & 63.5\% & 75.8\% & 95.9\% & 95.9\% & 95.4\% \\ \hline
Thompson Sampling & 5 $\times$ 10 & 96.8\% & 97.3\% & 79.0\% & 97.2\% & 95.1\% & 76.1\% \\ \hline
Ridge Regression & 5 $\times$ 10 & 65.0\% & 67.8\% & 76.8\% & 97.5\% & 97.5\% & 97.0\% \\ \hline
Ridge Reg. + Perturbation & 5 $\times$ 10 & 70.4\% & 71.7\% & 81.0\% & 97.5\% & 97.5\% & 97.0\% \\ \hline
Known Parameter & 5 $\times$ 10 & 100.0\% & 100.0\% & 100.1\% & 98.2\% & 98.2\% & 98.2\% \\ \hline
 \hline
Least Squares & 10 $\times$ 5 & 71.4\% & 73.1\% & 81.7\% & 95.9\% & 95.9\% & 95.4\% \\ \hline
Thompson Sampling & 10 $\times$ 5 & 96.7\% & 97.7\% & 77.7\% & 96.8\% & 94.3\% & 74.6\% \\ \hline
Ridge Regression & 10 $\times$ 5 & 71.6\% & 75.0\% & 82.4\% & 97.3\% & 97.3\% & 96.8\% \\ \hline
Ridge Reg. + Perturbation & 10 $\times$ 5 & 76.4\% & 80.2\% & 83.3\% & 97.3\% & 97.3\% & 96.6\% \\ \hline
Known Parameter & 10 $\times$ 5 & 100.0\% & 100.0\% & 100.0\% & 97.8\% & 97.8\% & 97.8\% \\ \hline
 \hline
Least Squares & 10 $\times$ 10 & 59.0\% & 64.5\% & 71.0\% & 94.5\% & 94.2\% & 93.5\% \\ \hline
Thompson Sampling & 10 $\times$ 10 & 96.4\% & 89.0\% & 38.8\% & 96.0\% & 86.3\% & 40.5\% \\ \hline
Ridge Regression & 10 $\times$ 10 & 59.4\% & 65.2\% & 71.8\% & 96.8\% & 96.7\% & 95.2\% \\ \hline
Ridge Reg. + Perturbation & 10 $\times$ 10 & 68.9\% & 70.4\% & 73.0\% & 96.7\% & 96.6\% & 95.0\% \\ \hline
Known Parameter & 10 $\times$ 10 & 100.0\% & 100.0\% & 100.1\% & 97.5\% & 97.5\% & 97.5\% \\ \hline
 \hline
Least Squares & 25 $\times$ 25 & 44.5\% & 53.7\% & 67.1\% & 91.4\% & 91.2\% & 84.7\% \\ \hline
Thompson Sampling & 25 $\times$ 25 & 98.4\% & 18.5\% & 1.8\% & 92.3\% & 21.2\% & 2.7\% \\ \hline
Ridge Regression & 25 $\times$ 25 & 44.7\% & 54.7\% & 65.8\% & 95.3\% & 94.0\% & 83.4\% \\ \hline
Ridge Reg. + Perturbation & 25 $\times$ 25 & 65.8\% & 63.9\% & 69.6\% & 95.1\% & 94.0\% & 83.6\% \\ \hline
Known Parameter & 25 $\times$ 25 & 100.0\% & 100.0\% & 100.0\% & 96.2\% & 96.2\% & 95.9\% \\ \hline
 \hline
Least Squares & 25 $\times$ 50 & 41.6\% & 48.0\% & 58.0\% & 92.7\% & 92.7\% & 90.4\% \\ \hline
Thompson Sampling & 25 $\times$ 50 & 97.8\% & 46.3\% & 10.4\% & 95.4\% & 50.8\% & 11.8\% \\ \hline
Ridge Regression & 25 $\times$ 50 & 43.6\% & 49.5\% & 67.1\% & 96.4\% & 96.0\% & 91.1\% \\ \hline
Ridge Reg. + Perturbation & 25 $\times$ 50 & 57.7\% & 59.2\% & 71.3\% & 96.3\% & 96.0\% & 91.2\% \\ \hline
Known Parameter & 25 $\times$ 50 & 100.0\% & 100.0\% & 100.0\% & 97.4\% & 97.4\% & 97.4\% \\ \hline
 \hline
Least Squares & 50 $\times$ 25 & 49.3\% & 53.6\% & 58.8\% & 92.5\% & 92.8\% & 90.5\% \\ \hline
Thompson Sampling & 50 $\times$ 25 & 98.6\% & 44.8\% & 7.9\% & 94.8\% & 50.2\% & 10.3\% \\ \hline
Ridge Regression & 50 $\times$ 25 & 52.3\% & 55.1\% & 65.1\% & 96.1\% & 95.7\% & 91.3\% \\ \hline
Ridge Reg. + Perturbation & 50 $\times$ 25 & 63.9\% & 62.6\% & 69.9\% & 96.0\% & 95.7\% & 91.1\% \\ \hline
Known Parameter & 50 $\times$ 25 & 100.0\% & 100.0\% & 100.0\% & 97.0\% & 97.0\% & 97.1\% \\ \hline
 \hline
Least Squares & 50 $\times$ 50 & 43.2\% & 51.2\% & 59.5\% & 91.4\% & 91.5\% & 85.8\% \\ \hline
Thompson Sampling & 50 $\times$ 50 & 98.1\% & 13.2\% & 2.3\% & 93.1\% & 19.7\% & 3.5\% \\ \hline
Ridge Regression & 50 $\times$ 50 & 44.9\% & 52.9\% & 65.0\% & 95.6\% & 94.5\% & 84.9\% \\ \hline
Ridge Reg. + Perturbation & 50 $\times$ 50 & 59.3\% & 63.2\% & 67.7\% & 95.5\% & 94.4\% & 85.2\% \\ \hline
Known Parameter & 50 $\times$ 50 & 100.0\% & 100.0\% & 99.9\% & 96.7\% & 96.7\% & 96.8\% \\ \hline
    \end{tabular}
    \caption{All percentages shown are the average revenue over 100 simulations divided by the best average revenue achievable ($\mathrm{OPT}(\calP)$).   \label{Tab:AllRes10k} }
\end{table}

\end{document}